\documentclass{article}

\PassOptionsToPackage{numbers, compress}{natbib}

\usepackage[preprint]{neurips_2025}

\usepackage[utf8]{inputenc} 
\usepackage[T1]{fontenc}    
\usepackage{hyperref}       
\usepackage{url}            
\usepackage{booktabs}       
\usepackage{amsfonts}       
\usepackage{nicefrac}       
\usepackage{microtype}      
\usepackage{xcolor}         
\usepackage{mathtools, amsthm}

\usepackage{algorithm}
\usepackage{algpseudocode}
\usepackage{amsmath}       
\usepackage{amssymb}  
\usepackage{algorithmicx}

\newtheorem{theorem}{Theorem}[section]
\newtheorem{definition}[theorem]{Definition}
\newtheorem{corollary}[theorem]{Corollary}
\newtheorem{lemma}[theorem]{Lemma}
\newtheorem{claim}[theorem]{Claim}

\AfterEndEnvironment{theorem}{\noindent\ignorespaces}
\AfterEndEnvironment{definition}{\noindent\ignorespaces}
\AfterEndEnvironment{corollary}{\noindent\ignorespaces}
\AfterEndEnvironment{lemma}{\noindent\ignorespaces}
\AfterEndEnvironment{claim}{\noindent\ignorespaces}

\newcommand{\R}{\mathbb{R}}
\newcommand{\prob}{\mathbb{P}}

\newcommand{\E}{\mathbb{E}}

\newcommand{\pto}{\xrightarrow{p}}
\newcommand{\dto}{\xrightarrow{d}}
\usepackage{float}
\newcommand{\var}{\operatorname{Var}}
\newcommand{\N}{\mathcal{N}}

\title{Differentially private ratio statistics}

\author{%
  Tomer Shoham \\
  Department of Computer Science\\
  The Hebrew University of Jerusalem\\
  \texttt{Tomer.Shoham@mail.huji.ac.il} \\
  \And
  Katrina Ligett \\
  Department of Computer Science\\
  The Hebrew University of Jerusalem\\
  \texttt{Katrina.Ligett@mail.huji.ac.il} \\
}

\begin{document}

\maketitle

\begin{abstract}
Ratio statistics—such as relative risk and odds ratios—play a central role in hypothesis testing, model evaluation, and decision-making across many areas of machine learning, including causal inference and fairness analysis. However, despite privacy concerns surrounding many datasets and despite increasing adoption of differential privacy, differentially private ratio statistics have largely been neglected by the literature and have only recently received an initial treatment, by \citet{lin2024differentially}. This paper attempts to fill this lacuna, giving results that can guide practice in evaluating ratios when the results must be protected by differential privacy. In particular, we show that even a simple algorithm can provide
excellent properties concerning privacy, sample accuracy, and bias, not just asymptotically but also at quite small sample sizes. Additionally, we analyze a differentially private estimator for relative risk, prove its consistency, and develop a method for constructing valid confidence intervals. Our approach bridges a gap in the differential privacy literature and provides a practical solution for ratio estimation in private machine learning pipelines.
\end{abstract}

\section{Introduction}
Differential privacy (DP)~\cite{dwork2006calibrating} offers a rigorous framework and has become the gold standard both in academia and in industry for privacy-preserving data analysis. Ratio statistics are frequently of interest in private machine learning pipelines, whether the application is estimating the performance of a medical classifier or evaluating a fairness metric on a predictor in a loan-granting context.

Differential privacy intuitively works by perturbing computations or their results to mask small changes in the underlying data, and the magnitude of the necessary perturbation is driven by a quantity known as the sensitivity---a bound on how much the quantity of interest might change if one individual's data were to change. Some statistics, like counts or averages, are very robust to changes in one individual. For such statistics, simple techniques can ensure differential privacy while maintaining good accuracy. Ratios, however, are challenging for differential privacy---if the numerator or the denominator reflects a count of people with a certain property, then small changes in the underlying population can potentially dramatically change the ratio, which can result in huge (or even unbounded) sensitivity. 

However, all is not lost. In this work we show that even if an analyst does not have free reign to design a sophisticated algorithm for computing a differentially private ratio, but is merely a consumer of differentially private count statistics from which she constructs ratios, the resulting ratio statistics can have excellent properties concerning privacy, sample accuracy, and bias, not just asymptotically, but also at quite small sample sizes (Section~\ref{sec:noisy_lap_counts}). We achieve these results despite the moments of the noised ratio being undefined.

A key aspect of our findings (Section~\ref{subsec:connection})  is that although the resulting private ratio statistics can indeed be adversely affected by the introduction of privacy, the main adverse effects manifest in precisely the situations where the original, non-private ratio statistic had high variance and hence itself was not particularly meaningful for statistical tests.

We also show (Section~\ref{sec:num_study}) via numerical simulation that although the simple algorithm of post-processing noisy counts produces a biased estimator, this method is competitive with more complicated techniques designed to handle sensitive statistics, especially on small samples.

\subsection{Related work}
The release of frequency tables---a list of count queries with certain constraints---under differential privacy was first studied by~\citet{barak2007privacy}.
Our work follows a growing literature on hypothesis testing on binomial random variables under DP; for a survey of differential privacy and other privacy methods applied to frequency tables, see~\cite{rinott2018confidentiality} and references therein, particularly~\cite{vu2009differential,gaboardi2016differentially,awan2018differentially}. 

Our work also contributes to a literature that considers the challenges faced by an analyst who is a \emph{consumer} (rather than an active producer) of differentially private statistics (e.g., provided, already noised, by a statistical agency such as the Census) and wishes to post-process them and understand the downstream effect that privacy has introduced. The US Census Bureau's decision to enforce differential privacy on data products derived from 2020 Census data \citep{mervis2019can} has sparked a rich discussion of potential impacts of differential privacy~\citep{kenny2021impact}, including specifically on ratio statistics \citep{mueller20222020, santos2020differential}. In this regard, our results provide reassurance that the introduction of differential privacy likely imposes quite limited harm to ratio statistics, even at small sample sizes.

Most closely related to our work is a recent paper of Lin et al.~\cite{lin2024differentially}, wherein they study the problem of constructing confidence regions around proportion estimates when using stratified random sampling techniques, subject to guarantees of differential privacy. In situations where the denominator of their proportion estimate (the size of the relevant sample) is private information, they consider adding Gaussian noise for DP to both the numerator and the denominator. As a result, they end up with a mathematical object---a ratio of two Gaussian random variables---that is also central to our analysis. They provide a very nice analysis of this ratio and show a confidence interval that is valid asymptotically. To arrive at this result, they leverage results of D{\'\i}az-Franc{\'e}s and Rubio~\cite{diaz2013existence} on the goodness of approximation of the ratio of two Gaussian random variables, in particular a Berry-Esseen-like theorem bounding the total variation distance between the ratio of two Gaussians and a single Gaussian distribution. In Section~\ref{sec:CI}, we use similar proof techniques to \cite{lin2024differentially} to obtain confidence interval results on differentially private ratio statistics; our results complement but are rather incomparable to the results of \cite{lin2024differentially}, as our focus is not on asymptotic coverage but on finite-sample properties such as sample accuracy and bias.

\section{Preliminaries}\label{sec:2}

\subsection{A ratio of two counts}

We consider the ratio $Z = X/Y$ of two natural numbers, $X,Y\in {\mathbb N}$. In Section \ref{sec:DP_ratio_estimation_naive} and Section \ref{sec:noisy_lap_counts} we consider $X$ and $Y$ to be constants and our goal is to estimate $Z$; we denote the DP-estimate by $\ \widetilde{Z}$. In general,
we denote a private estimator with $\widetilde{\cdot}$ and the non-private estimator by  $\ \widehat{\cdot}$.

In Section \ref{sec:CI}, motivated by the widely used relative risk statistic, we treat $X$ and $Y$ as binomial random variables with parameters $(p_x, n_x),$ $(p_y,n_y)$, and focus on the statistical task of estimating $p_x/p_y$, which is known as the \textbf{Relative Risk} (RR) statistic. The RR is the ratio of the probability of an outcome (e.g., diseased relative to healthy) in the exposed group to the probability of that outcome in the non-exposed group. If the sample sizes of the groups are known, which is typically the case, then the relative risk estimate is proportional to a ratio of two independent counts. For further reading about the relative risk, see \cite{agresti2012categorical}.

\subsection{Differential privacy}

We give a very brief introduction to differential privacy. A more detailed version is given in Appendix \ref{App:DP_intro}; for further reading see \cite{dwork2014algorithmic} and references therein.

Let $\cal D$ be an abstract data domain. A dataset of size $n$ is a collection of $n$ individuals' data records: $D=\{D_i\}_{i=1}^n \in {\cal D}^n$. 
We assume that $n$ is public; that is, we do not protect the size of the dataset.
We call two datasets $D, D' \in {\cal D}^n$ neighbors, denoted by $D \sim D'$, if they are identical except in one of their records. If we consider a ratio $X/Y$ where both $X$ and $Y$ are counts of the number of individuals in a dataset that have a certain set of properties, then  $X$ and $Y$ will each change by at most one between any two neighboring datasets. 

\begin{definition}[Differential privacy]
        Given a data domain ${\cal D}$ and some domain of responses $\mathcal{R}$, we say that a mechanism ${\cal M}:{\cal D}^n \rightarrow \mathcal{R}$ satisfies $(\varepsilon, \delta)$-\textit{differential privacy}, denoted by  $(\varepsilon, \delta)$-DP, if 
        $\mathbb{P}({\cal M}(D) \in E)\le e^\varepsilon \mathbb{P}({\cal M}(D') \in E)+\delta $
        for all $D\sim D' \in {\cal D}^n$ and all $E \subseteq \mathcal{R}$.
\end{definition} 

Differential privacy has two very important properties. The first is that the composition of two differentially private algorithms is differentially private (see Theorem 3.14 in \cite{dwork2014algorithmic}). The second  is that differential privacy cannot be harmed by post-processing---that is, if an algorithm is differentially private, then no analyst without additional access to the dataset can degrade the privacy guarantee through further analysis of the algorithm's output (see Proposition 2.1 in \cite{dwork2014algorithmic}).

One way to achieve DP for algorithms that output numbers (or vectors of numbers) is by noise-addition mechanisms. Two such mechanisms are The Laplace Mechanism (see Definition \ref{Def:Lap_mec}) and the Gaussian Mechanism (Definition \ref{def:Gaus_mec}), both adding zero-mean Laplace or Gaussian noise to a statistic, with variance proportional to a quantity known as the global sensitivity (Definition \ref{def:global_sen}).

There are two main results on the relationship between the variance of the added noise and the differential privacy properties of the Gaussian mechanism.
The first, stated in \cite{dwork2014algorithmic}, is not tight but is explicit (see Theorem \ref{thm:Gaus_mec_dwork}). The second, stated in \cite{balle2018improving}, is tight, but the variance has to be approximated numerically (see Theorem \ref{thm:Gaus_mec_balle}). We use the first to state explicit results to give intuition and the second for our numerical simulations.

\subsection{Properties of the statistic}


When we study differential privacy, we want to measure the harm to the statistic with respect to the privacy-induced perturbation, given the sample. This is called \textit{sample accuracy}.

\begin{definition}[$(\alpha,\beta)$-sample accuracy]\label{def:sample_acc}
    Given a family of queries $Q$, we say that a mechanism ${\cal M}:{\cal D}^n\times Q \rightarrow \mathbb{R} $ is $(\alpha,\beta)$-sample accurate if for any dataset $D\in {\cal D}^n$ and for any query $q\in Q$ it outputs a value ${\cal M}(D,q) = r \in \mathbb{R}$ such that 
    $\mathbb{P}(|r-q(D)| > \alpha) < \beta$, where $\alpha \in \mathbb{R}^+,\beta \in [0,1]$.
\end{definition}

Generally, an analyst's goal in estimating a statistic is to gather information about the data-generating distribution and not just about the data sample at hand. Confidence intervals (CI) are a tool for expressing the level of uncertainty about an estimated parameter, where that uncertainty could stem from many sources, including sampling randomness or intentional perturbation for privacy protection.

 \begin{definition}[Approximate CI]\label{def:CI}
     Let $F_{\theta}$ be some distribution defined on $R$, and let $X_n=(x_1,...x_n)$ be an i.i.d random sample from this distribution. A $(1-\gamma)$-CI for the parameter $\theta$ is an interval $(u(X_n), v(X_n))$ where $u(X_n)$ and $v(X_n)$ are random variables, such that
    $ \lim_{n\rightarrow \infty}{\mathbb P}(u(X_n) \leq \theta \leq v(X_n)) = 1-\gamma, $
    where  randomness is with respect to $x_1,...x_n \overset{i.i.d}{\sim} F_{\theta}$, and also over the possible randomness of $u(\cdot)$ and $v(\cdot)$.
 \end{definition}

Throughout the paper, when we address CIs, we refer to this definition. A more conservative notion could require that for {\em every} $n$, the probability that the true parameter is in the CI is at least $1-\gamma$. The theoretical CI results we give are asymptotic, and we also numerically study what $n$ is sufficient.

\subsection{Other methods for DP-ratio estimation}\label{sec:DP_ratio_estimation_naive}

Our paper studies the simple method of adding noise to each count and post-processing to obtain a ratio. We briefly detail other possible methods for comparison, and refer the reader to the appendices for more details.

\paragraph{Direct perturbation}
We briefly give intuition for why direct perturbation of ratio statistics for DP yields terrible privacy-accuracy tradeoffs; we give more details in Appendix \ref{App:DP_ratio_naive}. The sensitivity of a ratio statistic is very large because a small change in a small denominator, with a large numerator, can cause a large change in the statistic (see Claim \ref{clm:naive_lap_dp}). A possible remedy, given that both $X,Y$ are strictly positive, is to take the log of the ratio as a pre-processing step, add noise that scales like the sensitivity of the log (which is 2; see Claim \ref{claim:lap_dp}), and exponentiate the resulting noisy statistic (post-processing). In Appendix \ref{App:DP_ratio_naive}, we derive the bias of this algorithm and explain how it can be debiased (see Claim \ref{clm:bias_log_RR}); we also derive guarantees on its sample accuracy (see Claim \ref{clm:acc_exp_naive}) and show that if $X$ and $Y$ grow at the same rate (that is, if $Z$ is fixed), then the accuracy guarantee achieved by this approach is a constant that does not improve with increased sample size, a quite weak guarantee.

\paragraph{Local-sensitivity-based methods}
The study of a statistic with high global sensitivity invites exploration of local-sensitivity-based privacy methods such as the smooth sensitivity framework~\citep{nissim2007smooth} and the propose-test-release framework~\citep{dwork2009differential}. At a high level, both attempt to tailor the noise addition to a more local version of the sensitivity. We consider them in Appendix \ref{App:local_noise_methods}, including how to tailor them for a ratio statistic (which might be of interest by its own), and compare the performance of our simple algorithm to these algorithms in Section \ref{sec:num_study}. 

\section{A ratio of noisy counts}\label{sec:noisy_lap_counts}
We consider an analyst who wishes to construct a ratio statistic from counts that have been protected by Laplace noise addition. In this section, we discuss the resulting privacy, accuracy, and bias, and we address bias correction. As we will see in the next section, our main interest is in finite-sample properties of the DP statistic that can yield guidance with regard to even very small samples. \citet{lin2024differentially} consider the more general problem of the release of a ratio of two noised statistics and focus on its asymptotic coverage properties (as we discuss further in Section~\ref{sec:CI}, where we treat coverage issues).

Although our paper focuses ratios of counts, the results in this section are much more general, extending to any ratio of noisy statistics with bounded sensitivity. 

\subsection{Unprocessed ratio of noisy counts}

Define the algorithm \textit{LaplaceNoisedCounts$_\varepsilon(X, Y)$}, which receives a privacy parameter $\varepsilon > 0$ and counts $X, Y > 0$, and simply adds independent Laplace noise to each count to return private estimates $\widetilde{X}=X + \mathrm{Lap}\left(\frac{2}{\varepsilon}\right)$ and $\widetilde{Y}=Y + \mathrm{Lap}\left(\frac{2}{\varepsilon}\right)$.

We consider an equal split of the privacy budget allocated to the two counts, both for simplicity and because one of our motivations is the perspective of an analyst who receives a pre-noised frequency table and thus does not directly control the privacy budget. One could also optimize over the split of $\varepsilon$ to get somewhat better results, as shown in \cite{swanberg2019improved}, where they analyze the F-statistic of ANOVA testing, which is a ratio, and perturb its denominator and numerator with different noise loads.

~\textit{LaplaceNoisedCounts$_\varepsilon(X, Y)$} is immediately differentially private, by composition.\footnote{We split the $\varepsilon$ and compose it in order to ensure that our analysis also holds under the more general notion of neighboring, wherein an individual is allowed to move arbitrarily between the cells of the frequency table.}
\begin{claim}
Algorithm~\textit{LaplaceNoisedCounts$_\varepsilon(X, Y)$} is $(\varepsilon,0)$-DP.
\end{claim}
There are good reasons to suspect that post-processing noisy counts might yield sub-optimal privacy-accuracy tradeoffs: the technique is extremely simple, and furthermore, this approach produces a biased estimate of the true ratio. (This can be seen directly from Jensen's inequality: the expectation of a ratio of two non-negative random variables is smaller than the ratio of the expectations.) Nonetheless, as we will see, this simple method performs quite well, even compared to more sophisticated algorithms.

We now establish a result on the sample accuracy of a ratio post-processed from the output of Algorithm~\textit{LaplaceNoisedCounts$_\varepsilon(X, Y)$}. Although the Claim is hard to interpret, it serves as a stepping stone towards Corollary \ref{clm:asym_case_counts}. 

\begin{claim}\label{clm:acc_Alg_count_no_max}
    Consider two natural numbers $X,Y \in {\mathbb N}$, and denote $Z\coloneqq X/Y \in {\mathbb R}^+$. For any $\varepsilon>0$ and $\alpha$ such that $0<\alpha<Z$, if we denote the output of Algorithm \textit{LaplaceNoisedCounts$_\varepsilon(X, Y)$} by $\widetilde{X}$, $\widetilde{Y}$ and take $\widetilde{Z}=\widetilde{X}/\widetilde{Y}$, then $\widetilde{Z}$ is $\left(\alpha, \beta \right)$-sample accurate (see Definition \ref{def:sample_acc}), with
\begin{multline*}
 \beta = \left(0.5 + \frac{0.5}{(Z-\alpha)^2-1} \right)\exp\left(\frac{-\varepsilon\alpha Y}{2(Z-\alpha) }\right)
 +\left(0.5+ \frac{0.5}{(Z+\alpha)^2-1} \right)\exp\left(\frac{-\varepsilon\alpha Y}{2(Z+\alpha) }\right) \\
  \hspace{0.65cm} -\frac{Z^2+\alpha^2-1}{(Z^2+\alpha^2-1)^2 -4\alpha^2 Z^2}\exp\left(\frac{-\varepsilon\alpha Y}{2}\right).
\end{multline*}
\end{claim}

The proof of Claim \ref{clm:acc_Alg_count_no_max} is technical, and appears in Appendix~\ref{sec:proofs}. Briefly,
we first derive the cumulative density function (CDF) of a ratio of two independent Laplace random variables with the same scale parameter (due to equal splitting of the privacy budget), but possibly different means (see Claim \ref{clm:Lap_ratio_cum_prob}). In general, it is not trivial that the expression has an explicit CDF (and, unfortunately, this is not true for a ratio of two independent Gaussian random variables). Then we move from the joint distribution to a conditional distribution (conditioned on $L_2$), which requires integrating over the support of $L_2$, that is, over $\R$. The conditional distribution has an explicit solution, and so does the marginal distribution of $L_2$. 

One takeaway from Claim \ref{clm:acc_Alg_count_no_max} is that, when $Z$ and $\alpha$ are constants, $\beta = O(e^{-\varepsilon Y})$; that is, this ``failure'' probability decreases exponentially with the privacy parameter $\varepsilon$, and with the denominator $Y$. Our empirical evaluation in  Figure \ref{fig:Sample_accuracy} shows that for $X, Y\geq50$, when $\varepsilon>1$, we have that with high probability ($1-\beta \approx 1$), the resulting ratio estimate is less than $\alpha=0.1$ away from the true value. More generally, an extensive numerical study demonstrates that this algorithm gives very good accuracy even for $X$ and $Y$ well below 100.

Although in this section we consider a fixed dataset, one might typically think of the counts $X$ and $Y$ as proportional to the sample size; when the sample size increases, $X$ and $Y$ also increase. This motivates the following Corollary:
\begin{corollary}\label{clm:asym_case_counts}
   Fix some $Z \in {\mathbb R}^+$. Consider $Y\in {\mathbb R}^+$ and let $X$ be $X=ZY$. For any $\varepsilon>0$ and $\alpha$ such that $0<\alpha<Z$, if we denote the output of Algorithm \textit{LaplaceNoisedCounts$_\varepsilon(X, Y)$} by $\widetilde{X}$, $\widetilde{Y}$ and take $\widetilde{Z}=\widetilde{X}/\widetilde{Y}$, then, $\widetilde{Z}$ converges in probability to the true value, $Z$. That is,
    $$  \lim_{Y \to \infty} \mathbb{P}(|\widetilde{Z}-Z| > \alpha) = 0. $$
\end{corollary}
This is immediate from Claim \ref{clm:acc_Alg_count_no_max}, since all the expressions in the exponents go to $-\infty$, and all other terms are constants. If we think of $X$ and $Y$ as sums of indicators for individuals drawn from some underlying distributions (see Section \ref{sec:CI} for more detail), this implies that the larger the sample size, the more accurate the output will be.

\subsection{Quantifying and correcting the bias}

While each noisy count is unbiased, the ratio itself is biased. This is intuitive from Jensen's inequality: the expectation of a ratio of two independent positive random variables is smaller than the ratio of the expectation of each random variable. However, it is known that the ratio of two Laplace or Gaussian random variables has no moments \citep{marsaglia2006ratios, lin2024differentially}, and thus taking a ratio of the outputs of \textit{LaplaceNoisedCounts$_\varepsilon(X, Y)$} does not have an expectation; formally we therefore cannot say anything about its ``bias.'' To address this, we can post-processes the noisy denominator to bound it to be at least 1.\footnote{One could also post-process the noisy numerator in the same spirit. Sample accuracy would improve, but the bias of the resulting statistic would increase. Either way, in most cases, the impact would be negligible.} When $Y$ and $\varepsilon$ are sufficiently large, the effect of the $\max$ term is negligible. For example, when $Y\geq 30$ and  $\varepsilon=0.1$, then $\prob(\widetilde{Y}<1)=0.0027$;  for $\varepsilon=1$ and $Y=10$, then $\prob(\widetilde{Y}<1)<10^{-4}$. This thus provides a way to understand the bias induced by taking the ratio of two perturbed counts. 

\begin{claim}[Bias of maxed ratio of noisy counts]\label{clm:bias_ratio_counts}
    Given two numbers $X,Y \in {\mathbb N}$, and $\varepsilon>0$, denote the output of Algorithm \textit{LaplaceNoisedCounts$_\varepsilon(X, Y)$} by $\widetilde{X}$, $\widetilde{Y}$, set $X/Y=Z \in \mathbb{R}^+$, and take $\widetilde{Z}=\widetilde{X}/\max(\widetilde{Y},1)$. Then, for any $\varepsilon>0$ we have
    \small{
    \begin{equation*}\begin{split}
        \mathbb{E}\left[\widetilde{Z}\right] 
       &= X \Bigg( \frac{1}{2}\exp\left(\frac{\varepsilon(1-Y)}{2} \right) +  \frac{\varepsilon}{4}\exp\left(-\frac{\varepsilon Y}{2}\right)  \left(Ei\left(\frac{\varepsilon Y}{2}\right)-  Ei\left(\frac{\varepsilon}{2}\right)\right)    -  \frac{\varepsilon}{4}\exp\left(\frac{\varepsilon Y}{2}\right)Ei\left(-\frac{\varepsilon Y}{2}\right) \Bigg) \\
       &\approx Z \cdot \frac{\varepsilon^2 Y^2}{4} \int_{-\infty}^{0} \frac{-e^u}{u^2 - \left(\frac{\varepsilon Y}{2}\right)^2} \, du,
    \end{split}\end{equation*}
    }
    where $Ei(x)= -\int_{-x}^{\infty}\frac{e^{-t}}{t}dt = \int_{-\infty}^{x}\frac{e^t}{t}dt$.
\end{claim}

The full proof of Claim \ref{clm:bias_ratio_counts} is given in Appendix~\ref{sec:proofs}. The $Ei$ function is known not to have an explicit solution (see \cite{abramowitz1988handbook}, Chapter 5), but it can be numerically approximated. 
When $\varepsilon Y$ is large enough (see comment below), the last integral satisfies
$
\int_{-\infty}^{0} \frac{-e^u}{u^2 - \left(\frac{\varepsilon Y}{2}\right)^2} \, du \approx \frac{1}{\left( \frac{\varepsilon Y}{2} \right)^2} \int_{-\infty}^{0} e^u \, du = \frac{4}{\varepsilon^2 Y^2}
$; that is, the bias decreases with $Y$ and $\varepsilon$.

In Section \ref{sec:num_study}, and specifically Figure \ref{fig:Bias_figure}, we empirically study the bias of the private and non-private versions of this estimator and show that the approximation in Claim \ref{clm:bias_ratio_counts} is very accurate for a wide range of parameters. Furthermore, we see that the bias is numerically small (a few percent of the true ratio) and it is negligible when $\varepsilon Y > 10$.

The bias derived in Claim \ref{clm:bias_ratio_counts} is a function of the specific dataset; there is no universal constant that corrects the bias for every dataset. However, an analyst with some prior or null hypothesis about the data-generating distribution (for example $p_x=p_y$, which is usually the null hypothesis when studying the relative risk) can use Claim \ref{clm:bias_ratio_counts} to quantify the bias by weighting the bias according to the corresponding probabilities.

\section{Consistency and confidence interval of the relative risk statistic}\label{sec:CI}

In Sections~\ref{sec:2} (and in more detail in Appendix \ref{App:DP_ratio_naive}) and~\ref{sec:noisy_lap_counts}, we studied accuracy solely with respect to the impact of the perturbation noise (see Definition \ref{def:sample_acc}).
  In practice, analysts often view data as generated by a random process, model that process, and study the model parameters and derive CIs for the parameter of interest. In this section, we study the combined effect of such sampling noise and the privacy noise on CIs. 

We suppose that  $X$ and $Y$ are independent binomially distributed random variables; that is, \begin{equation}\label{eq:bin_rv}
    X\sim Bin(n_x,p_x),\  Y\sim Bin(n_y,p_y),\text{ with } p_x,p_y \in (0,1),\text{ and } n_x,n_y \in \mathbb{N}, \ \  n_x+n_y=n.
\end{equation}
 We assume that $n_x$ and $n_y$ are non-private constants, and we wish to estimate $p \coloneqq p_x/p_y$, which is the relative risk under these assumptions. Note that if $n_x=n_y$, then the ratio of counts is exactly the relative risk statistic, and the results of Section \ref{sec:noisy_lap_counts} apply.
We note that one could study different distributional assumptions on $X$ and $Y$, including some dependence between them. Different modeling would require the development of additional methods to construct valid CIs.

We denote
    $\widehat{p}_x \coloneqq X/n_x, \ \ \widehat{p}_y \coloneqq Y/n_y \ \  \widetilde{p}_x \coloneqq \widetilde{X}/n_x, \ \  \widetilde{p}_y \coloneqq \widetilde{Y}/n_y, \ \ \widehat{p} \coloneqq \widehat{p}_x/\widehat{p}_y, \ \ \widetilde{p} \coloneqq \widetilde{p}_x/\widetilde{p}_y$.
That is, $\widehat{p}_x$ and  $\widehat{p}_y$ are the classical MLE estimators of $p_x$ and $p_y$. The estimators $\widetilde{p}_x$ and $ \widetilde{p}_y$ are their private perturbed counterparts (released with Laplace or Gaussian noise, for example).
We have the non-private estimator of the relative risk $\widehat{p}$ and its private counterpart $\widetilde{p}$. Our goal in this section is to analyze $\widetilde{p}$ and construct valid CIs for $p=p_x/p_y$.

It is well-known that the non-private relative risk statistic, $\widehat{p}$, gives a biased but consistent estimator of the relative risk (see, for example, \cite{van2000asymptotic}); that is, $\widehat{p}\pto p$ when $n_x,n_y\rightarrow \infty$. We derive similar results for $\widetilde{p}$, a ratio of noised counts, in Claim \ref{clm:consist_est}.
\begin{claim}[Consistency of $\widetilde{p}$]\label{clm:consist_est}
    Let $X$ and $Y$ be as in \eqref{eq:bin_rv}. Let $L_x$ and $L_y$ be two zero-mean independent random variables. Denote $\widetilde{X} = X+L_x, \ \widetilde{Y} = Y+L_y$, and $\widetilde{p} \coloneqq \widetilde{p}_x/\widetilde{p}_y$. If $\var(L_x) = O(n_x)$ and $\var(L_y)=O(n_y)$, then $\widetilde{p}\pto p$.
\end{claim}
\vspace{-0.5cm}
\begin{proof} Note that
    $ \widehat{p}_x \pto p_x$, $ \widehat{p}_y \pto p_y$, $\frac{L_x}{n_x} \pto 0$, and $\frac{L_y}{n_y} \pto 0$. It thus follows that 
    $ \widetilde{p}_x \pto p_x \ \ \text{and} \ \ \widetilde{p}_y \pto p_y$. 
    As long as $p_y>0$, the claim follows from the Continuous Mapping Theorem (CMT).
\end{proof}


\subsection{Confidence interval for the relative risk statistic}

Our method to construct a CI for a ratio is based on a normal approximation of the ratio. This method was proposed by~\cite{lin2024differentially}, and our proof of validity closely follows  theirs, but adapted to a ratio of two independent proportions with known sample sizes, instead of a single proportion with unknown sample size.

We first put forward our main theorem for the validity of a CI based on zero-mean-noised counts, and then explain some of the building blocks of the proof.

\begin{theorem}[Valid CI for the relative risk]\label{thm:Valid_CI}
    Let $X$ and $Y$ be as in \eqref{eq:bin_rv}. Let $L_x$ and $L_y$ be two independent zero-mean random variables, with variance $\var(L_x)=\var(L_y)=\Sigma$. Denote $\widetilde{X} = X+L_x, \ \widetilde{Y} = Y+L_y$, $\widetilde{p}_x \coloneqq \widetilde{X}/n_x$ and $\widetilde{p}_y \coloneqq \widetilde{Y}/n_y$, $p=p_x/p_y$ and $\widetilde{p} = \widetilde{p}_x/\widetilde{p}_y$.
    Define
    $$ \widetilde{V} = \widetilde{p}  \cdot \Bigg(\frac{1}{\widetilde{X}}-\frac{1}{n_x} + \frac{1}{\widetilde{Y}}-\frac{1}{n_y} \Bigg)^{1/2}.$$
    For any $\Sigma \in {\mathbb R}^{+}$ such that $\Sigma = O(n_x)$ and $\Sigma = O(n_y)$, if $n_x$ and $n_y$ both tend to infinity, then for any $0 < \alpha < 1$
        $$
    \prob\left(p \in \left(\widetilde{p} \pm \Phi^{-1}(1-\alpha/2)\cdot \widetilde{V}\right)\right) \rightarrow 1-\alpha, 
    $$
where $\Phi^{-1}(\cdot)$ is the inverse of the CDF of a standard Gaussian random variable.
\end{theorem}

Consequently, we can now derive CI for a ratio of noisy counts perturbed with the Gaussian or Laplace Mechanism, the most common methods to release counts.

\begin{corollary}[valid CI with the Laplace mechanism] If $L_x,L_y \overset{\text{i.i.d.}}{\sim} Lap(0, 2/\varepsilon)$ independent of $X$ and $Y$, then Theorem \ref{thm:Valid_CI} gives a valid CI for $p=p_x/p_y$, which is also $(\varepsilon,0)$-DP.  
\end{corollary}
\vspace{-0.3cm}
\begin{corollary}[valid CI with the Gaussian mechanism] If $L_x,L_y \overset{\text{i.i.d.}}{\sim} \N(0, \sigma^2)$ independent of $X$ and $Y$, where $\sigma^2$ is computed according to Theorem \ref{thm:Gaus_mec_balle} to obtain $(\varepsilon/2, \delta/2)$-DP, then Theorem \ref{thm:Valid_CI} gives a valid CI for $p=p_x/p_y$, which is also $(\varepsilon,\delta)$-DP.  
\end{corollary}

The proof of Theorem \ref{thm:Valid_CI} can be found in Appendix \ref{App:CIappendix}. In short, we claim that both $\widehat{p}_x$ and $\widehat{p}_y$ approach normal distributions by the CLT, and thus their private counterparts do, by Slutsky's Theorem. It follows that $\widetilde{p}$ approaches the distribution of a ratio of two normal random variables, by the CMT. We leverage results of \cite{diaz2013existence} to approximate the ratio of two Gaussian random variables with a Gaussian distribution in some constant interval containing $p$, and then asymptotically the normal approximation holds with probability one.

We compute coverage rates of the CIs derived using Theorem \ref{thm:Valid_CI} in Appendix~\ref{subsubsec:gauss_coverage}, for various parameters. For example, when we consider Laplace noise addition, when $\varepsilon=1$, $p_x=0.5$, $p_y=0.5$, $n_x=n_y=200$, taking $\alpha=0.05$, the empirical coverage is $0.933$, and, in general, not surprisingly, we have under-coverage. The Gaussian Mechanism, due to the non-negligible effect of $\delta$ at small sample sizes, has much worse coverage. The reason lies in the proof technique, which uses the fact that the perturbation noise is second-order, and thus asymptotically ignores it. This may be acceptable for large sample sizes, but our primary concern in this paper is small sample sizes. 

\paragraph{A more conservative CI for small sample sizes} To obtain a more conservative CI that is also useful for small sample sizes, instead of ignoring the noise asymptotically (using Slutsky's Theorem), we need to account for the perturbation error. To do so, we note that if the counts are perturbed with Gaussian noise (and not with arbitrary additive noise), then the variances of the binomial random variable (approaching a normal distribution) and the perturbation add up, and we get the following conservative CI.

\begin{theorem}[CI for the Gaussian Mechanism]\label{thm:Gaussian_Valid_CI}
   Let $X$ and $Y$ be as in \eqref{eq:bin_rv}. Let $L_x,L_y \overset{\text{i.i.d.}}{\sim} \N(0, \sigma^2)$ independent from $X$ and $Y$. Denote $\widetilde{X} = X+L_x, \ \widetilde{Y} = Y+L_y$, $\widetilde{p}_x \coloneqq \widetilde{X}/n_x$, $\widetilde{p}_y \coloneqq \widetilde{Y}/n_y$, $p=p_x/p_y$ and $\widetilde{p} = \widetilde{p}_x/\widetilde{p}_y$.
    Define
    $$ \widetilde{V} = \widetilde{p}  \cdot \Bigg(\frac{1}{\widetilde{X}}-\frac{1}{n_x} + \frac{1}{\widetilde{Y}}-\frac{1}{n_y} \\ + \sigma^2 \left(\frac{1}{\widetilde{X}^2} + \frac{1}{\widetilde{Y}^2} \right) \Bigg)^{1/2}.$$
    Then, for any $\sigma^2 \in {\mathbb R}^{+}$ such that $\sigma^2 = O(n_x)$ and $\sigma^2 = O(n_y)$, if $n_x$ and $n_y$ both tend to infinity, then for any $0 < \alpha < 1$,
        $$
    \prob\left(p \in \left(\widetilde{p} \pm \Phi^{-1}(1-\alpha/2)\cdot \widetilde{V}\right)\right) \rightarrow 1-\alpha, 
    $$
where $\Phi^{-1}(\cdot)$ is the inverse of the CDF of a standard Gaussian random variable.
\end{theorem}

Theorem \ref{thm:Valid_CI} gives intuition for the second-order effect the perturbation has on the standard deviation (see Subsection \ref{subsec:connection}) and provides us with a tool to generate more conservative CIs for other noise-addition mechanisms (such as Laplace, by plugging in the variance of the Laplace random variable $\sigma^2=2b^2$). Clearly, this is also valid by Theorem \ref{thm:Valid_CI}, since it only increases the CI. 

We empirically evaluate CIs for Laplace and Gaussian noise addition using Theorem \ref{thm:Gaussian_Valid_CI} in Figure \ref{fig:CI_width} and in Table \ref{tab:coverage-px-py}, for various parameters. We observe that when the actual counts are at least 30, this conservative method with Laplace noise performs quite well, and the empirical coverage is above the required coverage. A rule of thumb fairly common across statistics is that a sample size of 30 suffices for Gaussian approximation, and this seems to also hold in our case. 

To conclude, from Theorem \ref{thm:Valid_CI} we have that any noise-addition mechanism yields an asymptotically valid CI, and Theorem \ref{thm:Gaussian_Valid_CI} gives us intuition for how to account for the perturbation noise to get good coverage for small sample sizes. Both methods give similar results when the sample size is large. 

\subsection{The noise from privacy is problematic only when the noise from sampling is problematic}\label{subsec:connection} 

When an analyst works with DP statistics, one major concern is whether the perturbation introduced for privacy will ruin the ability to detect true discoveries. One interesting outcome of the CI we derived in Theorem~\ref{thm:Gaussian_Valid_CI} is that if the standard deviation of the ratio without perturbation is small (which corresponds to large $X$ and $Y$), then the variance added by the perturbation is also small; that is, the smaller the sampling randomness, the smaller the perturbation noise. Thus, if even before we consider privacy, the sampling randomness was very large, then we expect the privacy perturbation noise to be substantial, but this is also not an interesting case. If on the other hand an experiment was designed to have large power (the probability to detect true discoveries), then the increase in the CI due to noise for privacy should be reasonable.

\section{Numerical studies}\label{sec:num_study}

We present several numerical studies. First, following Section \ref{sec:noisy_lap_counts}, we explore sample accuracy for a variety of algorithms for privately releasing ratio statistics. We then study the bias of the ratio of noisy counts (accounting for both sampling randomness and privacy noise). Then, following Section \ref{sec:CI}, we study the width and coverage of the proposed CIs for the relative risk. Throughout the numerical study, we consider  $n_x=n_y$ because then the ratio of counts is exactly the relative risk. Since $X$ and $Y$ are independent, considering different $p_x$ and $p_y$ is similar to considering $n_x \neq n_y$. Due to space limitations, a more detailed description of the numerical studies can be found in Appendix~\ref{App:numerical_study}.

\begin{figure}[h]
\begin{center}
\centerline{\includegraphics[width=\columnwidth]{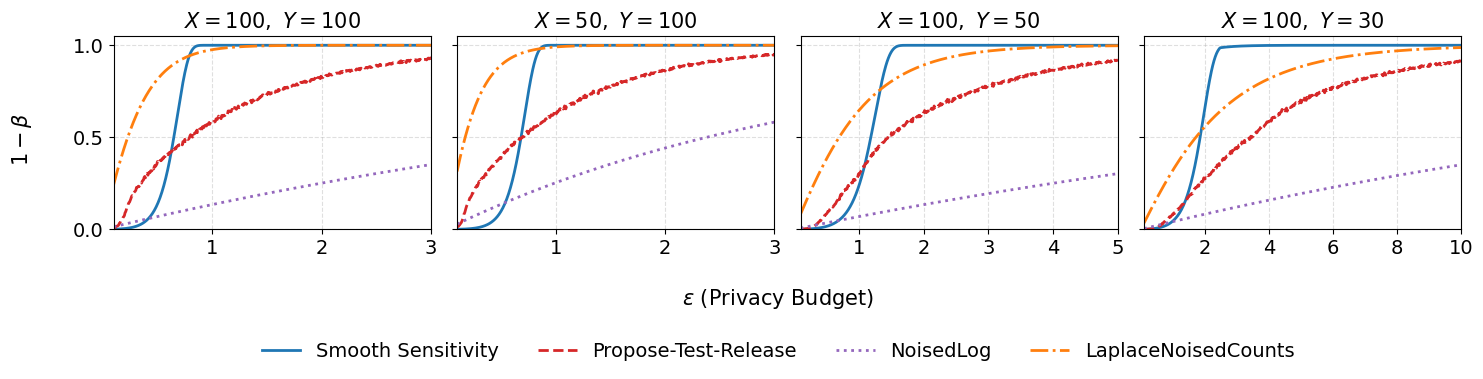}}
\caption{Comparison of the sample accuracy $(1 - \beta)$ of the algorithms LaplaceNoisedCounts 
(see Claim~\ref{clm:acc_Alg_count_no_max}),  
NoisedLog (see Claim~\ref{clm:acc_exp_naive}), Propose-Test-Release (see Algorithm~\ref{alg:PTR}), and the smooth sensitivity approach (described in Appendix \ref{App:local_noise_methods}). We take $\alpha=0.1$, and $\delta=1/n_x$ when needed.}
\label{fig:Sample_accuracy}
\end{center}
\vskip -0.2in
\end{figure}

\paragraph{Sample accuracy}

When $Y$ is small relative to $X$---that is, the local sensitivity is relatively large---the Smooth Sensitivity method performs well. That said, $\delta=1/n_x$ is extremely high (typically one sets $\delta<<1/n$). At more reasonable settings of $\delta$ and for these small sample sizes, \textit{LaplaceNoisedCounts$_\varepsilon(X, Y)$} outperforms the other algorithms for almost every $\varepsilon$. 

\paragraph{The bias of LaplaceNoisedCounts}
    In Figure \ref{fig:Bias_figure} we quantify the bias of the \textit{LaplaceNoisedCounts$_\varepsilon(X, Y)$} algorithm. We take $n_x=n_y$ so $X/Y=\widehat{p}_x/\widehat{p}_y=\widehat{p}$ and $\widetilde{X}/\widetilde{Y}=\widetilde{p}_x/\widetilde{p}_y=\widetilde{p}$. We compare the bias of $\widehat{p}$ (the non-private estimate), $\widetilde{p}$ (the private estimate), and Claim \ref{clm:bias_ratio_counts} about the bias of the private estimate. First, we see that the approximation derived in Claim \ref{clm:bias_ratio_counts} is highly accurate. Additionally, at a very reasonable privacy budget, the bias introduced by privacy perturbation is smaller than the bias introduced by sampling. Overall, it also seems that the bias is not significant.

\begin{figure}[h]
\begin{center}
\centerline{\includegraphics[width=\columnwidth]{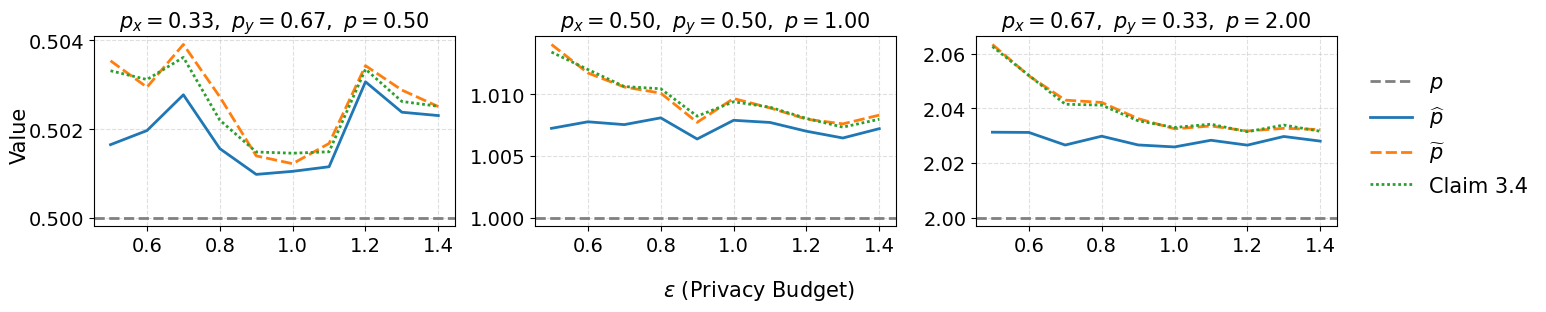}}
\caption{Comparison of: the bias of a ratio of noisy counts drawn according to Equation~\eqref{eq:bin_rv} ($\widehat{p}$), a ratio of the outputs of \textit{LaplaceNoisedCounts$_\varepsilon(X, Y)$} on those counts ($\widetilde{p}$),  and the expected value according to Claim \ref{clm:bias_ratio_counts}.
For each $\varepsilon$ we sample $20,000$ pairs of binomial random variables with $n_x=n_y=150$ and $p_x$ and $p_y$ as stated in the title, and report the average for each $\varepsilon$.} 
\label{fig:Bias_figure}
\end{center}
\vskip -0.2in
\end{figure}

\paragraph{Confidence interval width and coverage}
Following Section \ref{sec:CI}, we compare two non-private methods to construct CIs (for non-private counts), to the methods derived in Theorem \ref{thm:Valid_CI} and in Theorem \ref{thm:Gaussian_Valid_CI}.

\begin{figure}[h]
\begin{center}
\centerline{\includegraphics[width=\columnwidth]{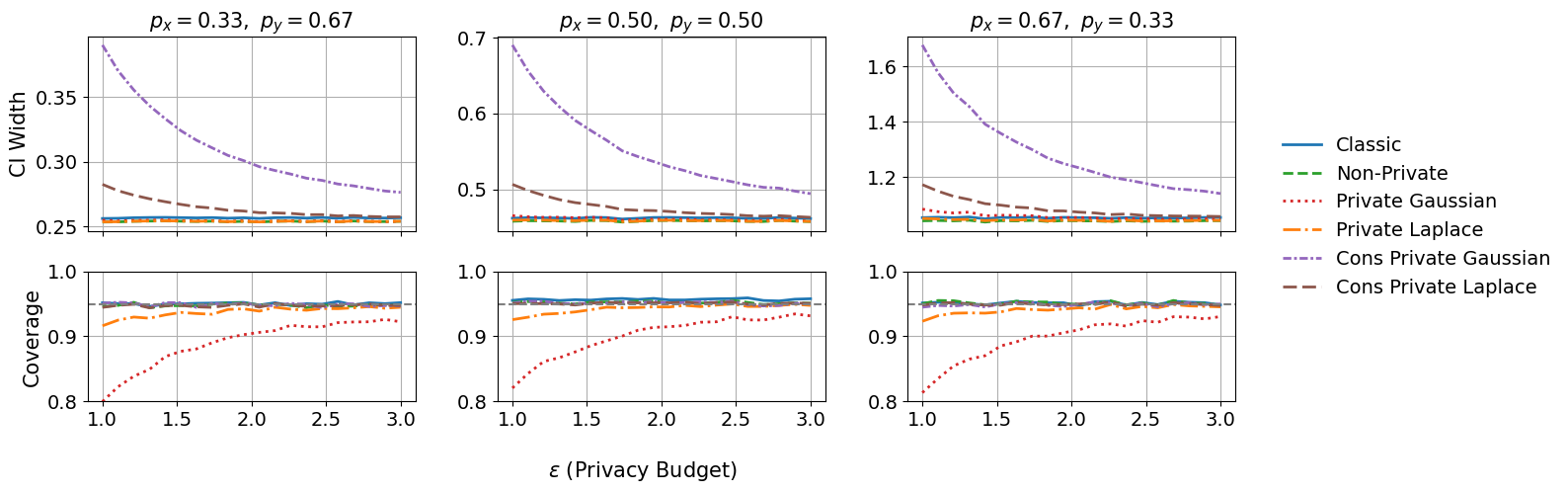}}
\caption{Average width and coverage of 10,000 CIs constructed for 10,000 pairs of binomial random variables drawn according to Equation~\eqref{eq:bin_rv}, with $n_x=n_y=150$. We compute the Classic CI (see Equation \eqref{eq:conf_int_RR}) and the non-private CI from binomial counts. We then perturb the counts with both Laplace and Gaussian noise to obtain $(\varepsilon,0)$-DP and $(\varepsilon, \delta=10^{-4})$-DP respectively, and compute CIs using Theorem \ref{thm:Valid_CI} (Private Laplace and Private Gaussian), and according to Theorem \ref{thm:Gaussian_Valid_CI} (Cons Private Gaussian and Laplace); for the Laplace noised counts we plug in the variance of the Laplace.}
\label{fig:CI_width}
\end{center}
\vskip -0.2in
\end{figure}

First, we see in Figure~\ref{fig:CI_width} that most of the CIs have the desired coverage, not just asymptotically (as Theorem \ref{thm:Valid_CI} and \ref{thm:Gaussian_Valid_CI} state), but also at small sample sizes. We see, for example, that the conservative private Laplace CI is about 20\% larger than the non-private CI when $\varepsilon=1$, while obtaining the desired coverage (note that $\E[X]$ and $\E[Y]$ are between 50 and 100). The Gaussian Mechanism underperforms the Laplace Mechanism 
due to the non-negligible effect of $\delta$  for small sample sizes.

\begin{ack}
This work was supported in part by a gift to the McCourt School of Public Policy and Georgetown University, Simons Foundation Collaboration 733792, Israel Science Foundation (ISF) grant 2861/20, Apple, ERC grant 101125913, and a grant from the Israeli Council of Higher Education. Views and opinions expressed are however those of the authors only and do not necessarily reflect those of the European Union or the European Research Council Executive Agency. Neither the European Union nor the granting authority can be held responsible for them.
\end{ack}

\newpage
\appendix

\section{Additional background on differential privacy}\label{App:DP_intro}

Let $\cal D$ be an abstract data domain. A dataset of size $n$ is a collection of $n$ individuals' data records: $D=\{D_i\}_{i=1}^n \in {\cal D}^n$. 
We assume that $n$ is public; that is, we do not protect the size of the dataset.
We call two datasets $D, D' \in {\cal D}^n$ neighbors, denoted by $D \sim D'$, if they are identical except in one of their records. 

One way to achieve DP for algorithms that output numbers (or vectors of numbers) is by noise addition mechanisms. In order to define them, we first define a quantity that is called \textit{Global Sensitivity} see Definition \ref{def:global_sen}, that measures the maximal change an output of a query can change (in some norm), when we change one individual, for any neighboring datasets.

\begin{definition}[Global sensitivity]\label{def:global_sen}
    Given a data domain ${\cal D}$ and a function $f: {\cal D}^n\rightarrow \mathbb{R}$, the global sensitivity of $f$ is given by
    $\Delta_{f} = \underset{\substack{D,D'\in {\cal D}^n \\ 
    D \sim D'}}{\text{max}}|f(D)-f(D')|.$
\end{definition}

The Laplace mechanism (see Definition \ref{Def:Lap_mec}) is one of the classic methods to obtain $(\varepsilon,0)$-DP. Simply out, it adds zero-mean Laplace noise to a statistic, with variance proportional to the global sensitivity. 

\begin{definition}[Laplace mechanism]\label{Def:Lap_mec}
   Consider a data domain ${\cal D}$ and a function $f: {\cal D}^n\rightarrow \mathbb{R}$. The Laplace mechanism, denoted by ${\cal M}^{Lap}_f$, simply adds independent Laplace noise to the results of $f$ on a dataset; that is,
    $$ {\cal M}^{Lap}_f(D) = f(D) + Y, \text{ where } Y \sim \text{Lap}(b), \ D \in {\cal D}^n,$$
   where $\text{Lap}(b)$ denotes a distribution with probability density function $ p(x) = \frac{1}{2b}\exp\left(-\frac{|x|}{b}\right).$ 
\end{definition}

\begin{theorem}[Theorem 3.6 in \cite{dwork2014algorithmic}]\label{clm:Lap_DP}
    Given some $\varepsilon>0$, the Laplace Mechanism with $b=\Delta_f/\varepsilon$ is $(\varepsilon,0)$-DP, where $\Delta_f$ is the global sensitivity of $f$ (see Definition \ref{def:global_sen}).
\end{theorem} 

The Gaussian Mechanism (Definition \ref{def:Gaus_mec}), similar to the Laplace Mechanism, adds zero-mean Gaussian noise to a statistic, with variance proportional to the global sensitivity.

\begin{definition}[Gaussian mechanism]\label{def:Gaus_mec}
    Consider a data domain ${\cal D}$ and a function $f: {\cal D}^n\rightarrow \mathbb{R}$. The Gaussian mechanism, denoted by ${\cal M}^{Gaus}_f$, simply adds independent Gaussian noise to the results of $f$ on a dataset; that is,
    $$ {\cal M}^{Gaus}_f(D) = f(D) + Y, \text{ where } Y \sim \text{N}(0, \sigma^2), \ D \in {\cal D}^n.$$
\end{definition}

There are two main results on the relationship between the variance of the added noise and the differential privacy properties of the Gaussian mechanism.
The first, stated in \cite{dwork2014algorithmic}, is not tight but is explicit (see Theorem \ref{thm:Gaus_mec_dwork}).

\begin{theorem}[Theorem A.1. in \cite{dwork2014algorithmic}]\label{thm:Gaus_mec_dwork}
   Let $f: {\cal D}^n \rightarrow \mathbb{R}$ be a function with global sensitivity $\Delta_f$. For any $\varepsilon \in (0,1)$, $\delta \in (0,1)$, and $D \in {\cal D}$, the Gaussian mechanism ${\cal M}_f^{Gaus}(D)=f(D)+Z$ where $Z\sim \N(0, \sigma^2)$ is $(\varepsilon,\delta)$-DP with 
   $\sigma^2 = \frac{2ln(1.25/\delta)\Delta_f^2}{\varepsilon^2}.$
\end{theorem} 

The second, stated in \cite{balle2018improving}, is tight, but the variance has to be approximated numerically (see Theorem \ref{thm:Gaus_mec_balle}). 

\begin{theorem}[Theorem 8 in \cite{balle2018improving}, simplified]\label{thm:Gaus_mec_balle}
    Let $f: {\cal D}^n \rightarrow \mathbb{R}$ be a function with global sensitivity $\Delta_f$. For any $\varepsilon>0$, $\delta \in (0,1)$, and $D\in {\cal D}$, the Gaussian  mechanism ${\cal M}^{Gaus}(D)=f(D)+Z$ with $Z\sim \N(0, \sigma^2)$ is $(\varepsilon,\delta)$-DP if and only if
    $$\Phi \left(\frac{\Delta_f}{2\sigma}-\frac{\varepsilon \sigma}{\Delta_f} \right) - e^{\varepsilon}\Phi \left(-\frac{\Delta_f}{2\sigma}-\frac{\varepsilon \sigma}{\Delta_f} \right) \leq \delta, $$
    where $\Phi$ is the cumulative distribution function of the standard Gaussian random variable.
\end{theorem}

\section{Naive methods for ratio perturbation}\label{App:DP_ratio_naive}

The most naive method to release a differentially private ratio of counts is by straightforward noise addition to the ratio. The numerator and denominator are finite and strictly positive (by assumption), so the sensitivity is finite. We consider Laplace noise addition for simplicity and because its cumulative density function has an explicit expression; similar results also hold for Gaussian noise.

Define the algorithm \textit{NaiveRelativeRisk$_\varepsilon(X, Y)$}, which receives privacy parameter $\varepsilon > 0$ and counts $X, Y > 0$, and returns a private estimate of the relative risk. It outputs $
\widetilde{Z} = \frac{X}{Y} + \mathrm{Lap}\left(\frac{n_x}{2\varepsilon}\right)$
where $n_x\in \mathbb{N}$, $n_x\geq X$, and $\mathrm{Lap}(\cdot)$ denotes Laplace noise with mean zero, with the given scale.

The algorithm is trivially differentially private, and we can derive its sample accuracy from the Laplace distribution.

\begin{claim}\label{clm:naive_lap_dp}
     Algorithm \textit{NaiveRelativeRisk$_\varepsilon(X, Y)$} is $(\varepsilon,0)$-DP.
\end{claim}
\begin{proof}
    Remember that $X \in [1,2,...,n_x]$ and $Y \in [1,2,...,n_y]$. The global sensitivity is thus the maximal change in absolute value between the ratios if we change one individual, either in $X$ or $Y$. This is given when the first sample is given by $X=n_x$, $Y=1$, and the neighbor sample is given by $X=n_x$, $Y=2$. Thus $$ \Delta = |n_x-n_x/2| = n_x/2,$$   and the claim follows from the Laplace mechanism (see Claim \ref{clm:Lap_DP}).
\end{proof}

\begin{claim}\label{clm:naive_lap_sample_acc}
    Algorithm \textit{NaiveRelativeRisk$_\varepsilon(X, Y)$} is $\left(\alpha,\beta \right)$-sample accurate, with
    $ \beta =  \exp \left(-2\alpha\varepsilon/n_x \right). $
\end{claim}

We note that the accuracy decreases with $n_x$, an undesirable property. Assuming a lower bound on $Y$ could help, but we prefer to avoid such assumptions if possible. If we assume that $Y \geq C$, we get a smaller sensitivity $ \Delta = n_x/(C^2+C)$. If $C$ is large enough or $n_x$ is small, this can work to some extent, but we want to avoid making such assumptions because we are interested in small sample sizes.

Claims \ref{clm:naive_lap_dp} and \ref{clm:naive_lap_sample_acc} can be easily adjusted to Gaussian noise; just remember that the CDF of the Gaussian distribution does not have a closed form, although many tail bounds exist. 

\subsubsection*{Perturbation of the log relative risk}

Another naive option is to perturb the logarithm of the relative risk ratio. By doing so we move from a multiplicative term to an additive term, add unbiased noise calibrated to it, and then post-process the perturbed response by exponentiating. 

For further discussion of tailored noise addition, see \cite{shoham2022asking}, where they show that for different metrics under the same noise mechanism, different transformations can be applied to optimize the accuracy of the output under the same $(\varepsilon,\delta)$-DP level.

Define the algorithm \textit{NoisedLog$_\varepsilon(X, Y)$}, which receives a privacy parameter $\varepsilon > 0$ and counts $X, Y > 0$, and returns a private estimate of the relative risk. It outputs $
\widetilde{Z} = \frac{X}{Y} \cdot e^{\mathrm{Lap}\left(\frac{\log(2)}{\varepsilon}\right)} $.

\begin{claim}\label{claim:lap_dp}
    Algorithm \textit{NoisedLog$_\varepsilon(X, Y)$} is $(\varepsilon,0)$-DP.
\end{claim}
\begin{proof}
    Think of the mechanism as adding Laplace noise to $log(X) - log(Y)$, and then taking the exponent. The global sensitivity of the quantity $log(X) - log(Y)$ is $\log(2)$ because the log is a concave function and $X,Y \geq 1$ by assumption. Using the fact that the Laplace mechanism is DP (Claim \ref{clm:Lap_DP}), and post-processing completes the proof.
 \end{proof}

\begin{claim}\label{clm:acc_exp_naive}
    Given two counts, $X,Y \in {\mathbb N}$, define $Z \coloneqq X/Y \in \mathbb{R}^+$. Then, for any $0<\alpha<Z$, Algorithm \textit{NoisedLog$_\varepsilon(X, Y)$} is $\left(\alpha, \beta \right)$-sample accurate, with
    $$ \beta = \frac{1}{2}\left(\frac{\alpha}{Z}+1 \right)^{-\frac{\varepsilon}{\log(2)}} + \frac{1}{2}\left(1-\frac{\alpha}{Z} \right)^{\frac{\varepsilon}{\log(2)}}.$$
\end{claim}
\begin{proof}
    We denote the output of \textit{NoisedLog$_\varepsilon(X, Y)$} by a random variable $\widetilde{Z}$. We further denote $L \sim \text{Lap}(\frac{\log(2)}{\varepsilon})$.
\begin{equation*}\begin{split}
     \mathbb{P}(|\widetilde{Z}-Z| \geq \alpha) &= \mathbb{P}\left(\big|Z \cdot \exp(L)-Z\big| \geq \alpha \right) \\
     &= \mathbb{P}\left(\big|\exp(L)-1\big| \geq \frac{\alpha}{Z} \right) \\
     &= \mathbb{P}\left(\exp(L) \geq \frac{\alpha}{Z}+1 \right) + 
     \mathbb{P}\left(\exp(L) \leq -\frac{\alpha}{Z}+1 \right) \\
     &\overset{(a)}{=}
     \mathbb{P}\left(L \geq \log\left(\frac{\alpha}{Z}+1\right) \right) + 
     \mathbb{P}\left(L \leq \log \left(1-\frac{\alpha}{Z}\right) \right)\\
     &= \frac{1}{2}\exp{\left(\frac{-\log(\frac{\alpha}{Z}+1)\varepsilon}{\log(2)} \right)} + \frac{1}{2}\exp{\left(\frac{\log(1-\frac{\alpha}{Z})\varepsilon}{\log(2)} \right)} \\
     &= \frac{1}{2}\left(\frac{\alpha}{Z}+1 \right)^{-\frac{\varepsilon}{\log(2)}} + \frac{1}{2}\left(1-\frac{\alpha}{Z} \right)^{\frac{\varepsilon}{\log(2)}} 
\end{split}\end{equation*}
{\footnotesize (a) We assumed that $\alpha<Z$}
\end{proof}
Note that the accuracy of Algorithm~\textit{NoisedLog$_\varepsilon(X, Y)$} is only a function of the ratio statistic and not of its components. If we increase $X$ and $Y$ but keep the same ratio, we get the same accuracy. The accuracy is also not affected by the sample size $n$.

\subsubsection*{Correcting the bias term}

Algorithm \textit{NaiveRelativeRisk$_\varepsilon(X, Y)$} produced an unbiased estimate because we added zero-mean noise. Algorithm \textit{NoisedLog$_\varepsilon(X, Y)$}, on the other hand, perturbs the ratio statistic with multiplicative noise, thus adding bias to the final result. 

\begin{claim}[Bias of \textit{NoisedLog$_\varepsilon(X, Y)$}]\label{clm:bias_log_RR}
    Denote the output of \textit{NoisedLog$_\varepsilon(X, Y)$} by $\widetilde{Z}$. When $\varepsilon > \log(2)$, multiplying $\widetilde{Z}$ by $1-(\log(2)/\varepsilon)^2$  yields an unbiased estimator for the ratio, $Z$.
\end{claim}

\begin{proof}
\begin{equation*}\begin{split}
    \mathbb{E} \left[\widetilde{Z} \right] = \mathbb{E} \left[e^{\log(Z)+Lap(\log(2)/\varepsilon)}\right]  = Z \cdot \mathbb{E}\left[e^{Lap(\log(2)/\varepsilon)}\right] = \frac{Z}{1-(\log(2)/\varepsilon)^2}.
\end{split}\end{equation*}
    The last equality is just the moment-generating function of the Laplace random variable, defined only for $\log(2)/\varepsilon<1$. In our case, the constraint becomes $\varepsilon > \log(2)$.
\end{proof}

 We note that the bias of Algorithm~\textit{NoisedLog$_\varepsilon(X, Y)$} is independent of the dataset; thus, correcting the bias does not require any further access to the dataset or to a prior. Note, however, that correcting the bias does not necessarily improve the sample accuracy.  Intuitively, this is because the noisy distribution has a long right tail, and hence the mass of the distribution is not around its expectation.

\newpage
\section{Local sensitivity based mechanisms}\label{App:local_noise_methods}

In Section \ref{sec:noisy_lap_counts}, we discussed the sample accuracy of a ratio computed from two perturbed counts with Laplace noise. We compared it with naive methods detailed in Appendix \ref{App:DP_ratio_naive}. In this section, we survey the classic local sensitivity-based methods and tailor them for a ratio statistic. The goal is to numerically quantify if, when, and by how much we lose by post-processing perturbed frequency tables, compared to releasing the sensitive data with perturbation explicitly tailored to the ratio statistic. We emphasize that these methods are designed to release only the ratio, thus, for the sake of CIs or hypothesis testing (which requires estimating the standard deviation) they are not helpful. Nevertheless, we wish to compare these methods to the noisy counts to understand the magnitude of the noise we add.

We explore two methods, one based on the smooth sensitivity framework~\cite{nissim2007smooth}, and the other based on the propose-test-release framework~\cite{dwork2009differential}. At a high level, they both attempt to tailor the noise addition to a more local version of the sensitivity (see Definition \ref{def:loc_sen}). 

Intuitively, the first method is based on finding a function that upper bounds the local sensitivity at any point, and is not sensitive to changes of individuals in the dataset. If these two requirements are met, revealing the smooth local sensitivity reveals very little information about a specific individual in the dataset, and adding noise proportional it that keeps the participants' privacy. The second method is based on \textbf{p}roposing a value for the local sensitivity, \textbf{t}esting privately whether this value is indeed larger than the true local sensitivity, if it is, \textbf{r}eleasing the value of the query privately with noise proportional to the proposed value if it was sufficient (and otherwise returning a failure message). Two  disadvantages of this second method are that it relies on access to a good prior on the local sensitivity, and that there is some probability that it will release ``fail'' rather than a statistic.

\subsection{Smooth sensitivity}\label{subsec_smo_sen}

We begin with some definitions.

\begin{definition}[Local sensitivity]\label{def:loc_sen}
    Let $f: {\cal D}^n \rightarrow \mathbb{R}$ and $D\in {\cal D}^n $. The local sensitivity of $f$ at $D$ is given by
    $$ LS_f(D) = \underset{\substack{D' \in {\cal D}^n \\ D \sim D'}}{\text{max}}|f(D)-f(D')|.$$
\end{definition}

Intuitively, adding noise that scales like the local sensitivity has privacy implications (because the magnitude of the noise leaks information about the data at hand), unlike adding noise that scales like the global sensitivity, which does not depend on the data at hand. That being said, if the noise is a function of the specific dataset and we are able to bound the effect on changing one individual for any dataset, adding noise that scales like it can be differentially private.

\begin{definition}[A smooth bound on the local sensitivity]
    For $\beta > 0$, a function $S : {\cal D}^n \rightarrow \mathbb{R}^+$ is a $\beta$-smooth upper bound
on the local sensitivity of $f$ if it satisfies the following requirements:
\begin{enumerate}
    \item $\forall D\in {\cal D}^n$: \quad $S(D) \geq LS_f(D)$
    \item $\forall D,D'\in {\cal D}^n, D\sim D'$: \quad $S(D)\leq e^\beta S(D')$
\end{enumerate}
\end{definition}

For two datasets $D,\tilde{D} \in {\cal D}^n$ we denote by $d(D,\tilde{D})$ the number of data elements we need to change in $D$ to get the dataset $\tilde{D}$. Trivially if $D \sim \tilde{D}$ we have $d(D,\tilde{D})=1$, and in general we have $d(D,\tilde{D})\leq n$.

\begin{definition}[Smooth sensitivity]\label{def:smooth_sen}
    For $\beta>0$, the $\beta$-smooth sensitivity of $f$ is
\begin{equation*}
 S^*_{f,\beta}(D) = \underset{\tilde{D} \in {\cal D}^n}{max}\left(LS_f(\tilde{D})e^{-\beta d(D,\tilde{D})} \right) \\ =  \underset{k=0,1,...,n}{\max}e^{-k\varepsilon}\left(\underset{\substack{\tilde{D} \in {\cal D}^n \\
d(D,\tilde{D})= k}}{\max} LS_f(\tilde{D}) \right).
\end{equation*}
\end{definition}

\begin{lemma}[Lemma 2.3 in \cite{nissim2007smooth}]
    $S^*_{f,\beta}(D)$ is a $\beta$-smooth upper bound on $LS_f(D)$. In addition, $S^*_{f,\beta}(D) \leq S(D)$ for all $D \in {\cal D}^n$ for every $\beta$-smooth upper bound on $LS_f(D)$.
\end{lemma}

\begin{lemma}[Lemma 2.5 in \cite{nissim2007smooth}, simplified]
    If $\beta \leq \frac{\varepsilon}{2ln(2/\delta)}$ and $\delta \in (0,1)$ the mechanism ${\cal M}(D) = f(D)+\eta$, where $\eta \sim Lap(\frac{2S(D)}{\varepsilon})$, is $(\varepsilon, \delta)$-differentially private.
\end{lemma}

\subsubsection{Smooth sensitivity and ratio of counts}

To simplify the analysis, assume that $1<X<n_x$, $1<Y<n_y$ (a version of the analysis that does not require the strict inequalities also holds). We consider that changing one individual changes at most one of the counts. The local sensitivity of a ratio of counts, $Z=X/Y$, (see Definition \ref{def:loc_sen}) is given by

\begin{equation*}\begin{split}
     LS_{Z}(D) &= \max \Bigg(
     \Bigg|\frac{X}{Y} - \frac{X+1}{Y}\Bigg|,
     \Bigg|\frac{X}{Y} - \frac{X-1}{Y}\Bigg|,  \Bigg|\frac{X}{Y} - \frac{X}{Y+1}\Bigg|,
     \Bigg|\frac{X}{Y} - \frac{X}{Y-1}\Bigg|
     \Bigg) \\
     &= \max \left(
     \frac{1}{Y}, \frac{X}{Y^2-Y}
     \right) \\
     &=  \frac{1}{Y}\begin{cases} 
      1 & X < Y \\
      \frac{X}{Y-1} & X \geq Y,
   \end{cases}
\end{split}\end{equation*}
 since $\Big|\frac{X}{Y} - \frac{X+1}{Y}\Big| = \Big|\frac{X}{Y} - \frac{X-1}{Y}\Big|$, and $\Big|\frac{X}{Y} - \frac{X}{Y+1}\Big| < \Big|\frac{X}{Y} - \frac{X}{Y-1}\Big|$.

With that in mind, we can derive the smooth sensitivity of our ratio (see Definition \ref{def:smooth_sen})

\begin{claim}[Smooth sensitivity of a ratio of counts]
Fix some constant $m\in \mathbb{N}$. Note that for datasets $D, D'$ such that $X, Y$ correspond to $D$ and $X', Y'$ correspond to $D'$, $d\left(D, D'\right)=|X-X'|+|Y-Y'|$. Then for all datasets $D'$,\\
\\
If $ X < Y-m $ :
    $$ \underset{D': d \left(D,D' \right)= m}{\max} LS_{Z}(D') = \frac{1}{(Y-m)}. $$
If $ X \geq Y-m $, and $Y-m > 1$:
    $$ \underset{D': d\left(D,D'\right)= m}{\max} LS_{Z}(D') = \frac{X}{(Y-m)(Y-m-1)}. $$
If $ X \geq Y-m $, and $Y-m < 2$:
    $$
        \underset{D': d\left(D,D'\right)= m}{\max} LS_{Z}(D') = \frac{\min \Big(n_x, X + m-Y+1\Big)}{2}.
    $$

\end{claim}

\begin{proof}
Let us examine the case $X < Y-m$. First, we note that if we decrease $X$, the local sensitivity does not change. If we increase $X$ by less than $m$, also, the local sensitivity does not change. So we can only increase the local sensitivity by decreasing $Y$. Because the local sensitivity is monotonically decreasing as long as $X < Y$ and we can decrease it by at most $m$, we got the first result. Now consider $X \geq Y$, thus we have $\frac{X}{Y-1}>1$. To increase the local sensitivity, we can either increase $X$, or decrease $Y$. It can easily be shown that decreasing $Y$ will increase the local sensitivity by more. Once $Y=1$, we can further increase $X$. Thus, to conclude, we first decrease $Y$ to 1 (if we can), and then increase $X$ to at most $n_x$.
\end{proof}

Unfortunately, at small sample sizes, the smooth sensitivity is large. Intuitively, the local-sensitivity decreases fast with $Y$, while the smooth sensitivity is only allowed to decrease logarithmically. At some point, when $Y$ is large enough, the smooth sensitivity is very close to the local sensitivity, showing very good performance of the algorithm.

\subsection{Propose-Test-Release}

\cite{dwork2009differential} put forward an algorithm (see Algorithm \ref{alg:PTR}) that can be used to estimate sensitive statistics if the analyst has some prior knowledge about the local sensitivity or does not care about the statistic if it happens to have a large local sensitivity.

\begin{algorithm}
\caption{Propose-Test-Release}\label{alg:PTR}
\begin{algorithmic}
\Require Privacy parameters $\varepsilon, \delta>0$, a query $f$, a sample $S$, a suggested bound on the local sensitivity $\beta$
\State $\gamma \gets$ distance (in data-elements) from S to the nearest dataset S' such that $LS_{f}(S') \geq \beta$.
\State $\hat{\gamma}$ = $\gamma$ + Lap($1/\varepsilon_1$)
\If{$\hat{\gamma} \leq \ln(1/\delta)/\varepsilon_1$}
    \State Return FAIL
\Else 
    \State Return $f(S) +$ Lap($\beta/\varepsilon_2$)
\EndIf
\end{algorithmic}
\end{algorithm}

\begin{theorem}
   The Propose-Test-Release algorithm described in Algorithm \ref{alg:PTR} is $(\varepsilon_1+\varepsilon_2, \delta)$-differentially private.
\end{theorem}

If the analyst were to guess the true local sensitivity, then $\gamma=0$. But then, with high probability (more than 0.5), the algorithm would return FAIL.
In analyzing the performance of propose-test-release, for simplicity, we will sum the probabilities of the two types of errors (the probability that the algorithm returns FAIL, plus the probability that the algorithm returns a value that is not sample-accurate). In some settings, it may, of course, be less harmful not to get a response than to get a response that is far away from the correct one. When we compare this algorithm to the ratio of noisy counts, we will give oracle access to the true local sensitivity (which is not private), and let the algorithm choose the optimal guess.

\newpage
\section{Confidence interval for the relative risk}\label{App:CIappendix}

In Sections~\ref{sec:2} (and in more detail in Appendix \ref{App:DP_ratio_naive}) and Section~\ref{sec:noisy_lap_counts}, we studied accuracy solely with respect to the impact of the perturbation noise (see Definition \ref{def:sample_acc}).
  In practice, analysts often view data as generated by a random process, model that process, and study the model parameters and derive CIs for the parameter of interest. In this section, we study the combined effect of such sampling noise and the privacy noise on CIs. 

We suppose that  $X$ and $Y$ are independent binomially distributed random variables; that is, \begin{equation*}
    X\sim Bin(n_x,p_x),\  Y\sim Bin(n_y,p_y),\text{ with } p_x,p_y \in (0,1),\text{ and } n_x,n_y \in \mathbb{N}, \ \  n_x+n_y=n.
\end{equation*}
 We assume that $n_x$ and $n_y$ are non-private constants, and we wish to estimate $p \coloneqq p_x/p_y$, which is the relative risk under these assumptions. Note that if $n_x=n_y$, then the ratio of counts is exactly the relative risk statistic, and the results of Section \ref{sec:noisy_lap_counts} apply.
We note that one could study different distributional assumptions on $X$ and $Y$, including some dependence between them. Different modeling would require the development of additional methods to construct valid CIs.

\subsection{Non-private confidence interval for the relative risk}

The classical method of constructing a CI for the relative risk is based on a direct normal approximation of the log-likelihood. \citet{katz1978obtaining} show that given independent $X \sim Bin(n_x,p_x)$ and $Y \sim Bin(n_y,p_y)$, if we denote $T = (X/n_x)/(Y/n_y)$, then $\log(T)$ is approximately normally distributed with mean $\log(p_x/p_y)$ and variance $((1/p_x) - 1)/n_x + ((1/p_y) - 1)/n_y$. Since $\E[X]=n_xp_x$, $\E[Y]=n_yp_y$, the estimated standard deviation of the log RR is given by
\begin{equation}\label{eq:std_log_RR}
    \widehat{\sigma}_{\log(T)} = \sqrt{\left(\frac{1}{X} - \frac{1}{n_x} \right) + \left(\frac{1}{Y} - \frac{1}{n_y} \right)},
\end{equation}
and a $(1-\alpha)$-CI (see Definition \ref{def:CI}) based on normal approximation, denoted by $CI_{\text{Classic}}$, is given by
\begin{equation}\label{eq:conf_int_RR}
  CI_{\text{Classic}} = \exp \Big(\ln(T) \pm \Phi^{-1}(1-\alpha/2) \cdot \widehat{\sigma}_{\log(T)} \Big). 
\end{equation}
where $\Phi^{-1}(1-\alpha/2)$ is the inverse of the CDF of a standard Gaussian random variable at the point $1-\alpha/2$.

We see from Equation~\eqref{eq:std_log_RR} that the smaller $X$ or $Y$ are, the larger the standard deviation is, which increases the width of the CI.

In this paper we use a different approach, which is more compatible with differential privacy, and, specifically, with nose addition to Binomial random variables. We compare our method, both coverage and width, to the classic method detailed in Equation~\eqref{eq:conf_int_RR}

\subsection{Private confidence interval for a ratio of noisy counts}\label{subsec:CI_Gaussian_noise}

We now consider a different method to construct a CI for the relative risk, based on a normal approximation of the ratio. This method was proposed by~\cite{lin2024differentially} for proportion estimation when the count and the sample size are perturbed with Gaussian noise. We extend it to a ratio of two proportions. First, we approximate each of the binomial random variables by a Gaussian random variable. We then approximate the ratio of these two Gaussians by a Gaussian distribution. This approach enables us to argue about the validity of a different private ratio of noisy counts.

\begin{theorem}[Theorem \ref{thm:Valid_CI}, restated]
     Let $X$ and $Y$ be as in \eqref{eq:bin_rv}. Let $L_x$ and $L_y$ be two independent zero-mean random variables, with variance $\var(L_x)=\var(L_y)=\Sigma$. Denote $\widetilde{X} = X+L_x, \ \widetilde{Y} = Y+L_y$, $\widetilde{p}_x \coloneqq \widetilde{X}/n_x$ and $\widetilde{p}_y \coloneqq \widetilde{Y}/n_y$, $p=p_x/p_y$ and $\widetilde{p} = \widetilde{p}_x/\widetilde{p}_y$.
    Define
    $$ \widetilde{V} = \widetilde{p}  \cdot \Bigg(\frac{1}{\widetilde{X}}-\frac{1}{n_x} + \frac{1}{\widetilde{Y}}-\frac{1}{n_y} \Bigg)^{1/2}.$$
    For any $\Sigma \in {\mathbb R}^{+}$ such that $\Sigma = O(n_x)$ and $\Sigma = O(n_y)$, if $n_x$ and $n_y$ both tend to infinity, then for any $0 < \alpha < 1,$
        $$
    \prob\left(p \in \left(\widetilde{p} \pm \Phi^{-1}(1-\alpha/2)\cdot \widetilde{V}\right)\right) \rightarrow 1-\alpha, 
    $$
where $\Phi^{-1}(\cdot)$ is the inverse of the CDF of a standard Gaussian random variable.
\end{theorem}

Before we begin the proof, we state the main Theorem of \citet{diaz2013existence}, which we leverage.

\begin{theorem}[Theorem 1 in \cite{diaz2013existence}]\label{thm:gauss_approx_ratio}
    Let $X$ be a normal random variable with mean $\mu_x$ and variance $\sigma^2_x$ and coefficient of variation $\delta_x = \sigma_x/\mu_x$ such that $0<\delta_x < \lambda \leq 1$, where $\lambda$ is a known constant. For every $\xi>0$, there exists $\gamma(\xi)\in (0, \sqrt{\lambda^2-\delta^2_x})$, and also a normal random variable $Y$, independent of $X$, with positive mean $\mu_y$, variance $\sigma^2_y$, and coefficient of variation $\delta_y = \sigma_y/\mu_y$ that satisfy the conditions
    \begin{equation}\label{eq:condition}
         0 < \delta_y \leq \gamma(\xi) \leq \sqrt{\lambda^2-\delta^2_x} \leq \lambda
    \end{equation}
    for which the following result holds:

    Any $z$ that belongs to the interval
    $$ I = \left[\beta - \frac{\sigma_z}{\lambda}, \beta + \frac{\sigma_z}{\lambda} \right]$$
    where $\beta = \mu_x / \mu_y$ and $\sigma_z = \beta\sqrt{\delta^2_x+\delta^2_y}$ satisfies that
    $$ |G(z)-F_Z(z)| < \xi,$$
    where $G(z)$ is the distribution function of a normal random variable with mean $\beta$ and variance $\sigma^2_z$, and $F_Z$ is the distribution function of $Z=X/Y$.
\end{theorem}

\begin{proof}
    We prove Theorem~\ref{thm:Valid_CI}. First, we prove that a ratio of noisy counts, when $n_x,n_y$ are large enough, is arbitrarily close to a ratio of two independent Gaussian random variables, and state their parameters.  

Let $X \sim \text{Binomial}(n_x, p_x)$, and define the empirical proportion $\widehat{p}_x := \frac{X}{n_x}$. By the classical Central Limit Theorem (CLT), we have
\[
\sqrt{n_x}(\widehat{p}_x - p_x) \xrightarrow{d} \mathcal{N}(0, \var(\widehat{p}_x)) \quad \text{as } n_x \to \infty,
\]
implying that \(\widehat{p}_x\) is asymptotically normal with mean \(p_x\) and variance $\var(\widehat{p}_x)$.
Let $L_x$ be a random variable that satisfies $\mathbb{E}[L_x] = 0$ and $\var(L_x)=\Sigma$ such that $L_x$ is independent of $X$. If $\Sigma=O(n_x)$ then $\frac{L_x}{n_x} \pto 0$ as $n_x \to \infty$. Define the perturbed estimator $\widetilde{p}_x := \widehat{p}_x + \frac{L_n}{n_x}$. Under the above conditions, we also have that \(\widetilde{p}_x\) is asymptotically normal with mean \(p_x\) and variance $\var(\widehat{p}_x)$; that is, for any $\eta > 0$, there exists $n_0 = n_0(\eta)$ such that for all $n_x > n_0$ and all $v \in \mathbb{R}$,
\begin{equation}\label{eq:p_x_tilde_to_gaus}
\left| \mathbb{P}(\widetilde{p}_x \le v) - G_x(v) \right| < \eta,
\end{equation}
where $G_x$ is the CDF of $\N(p_x, \var(\widehat{p}_x))$.

We can derive the same results for $Y$, and get that for any $\eta > 0$, there exists $n_0 = n_0(\eta)$ such that for all $n_y > n_0$ and all $v \in \mathbb{R}$,
\begin{equation}\label{eq:p_y_tilde_to_gaus}
\left| \mathbb{P}(\widetilde{p}_y \le v) - G_y(v) \right| < \eta,
\end{equation}
where $G_y$ is the CDF of $\N(p_y, \var(\widehat{p}_y))$.

We now consider the ratio $\widetilde{p} = \widetilde{p}_x/\widetilde{p}_y$. From \eqref{eq:p_x_tilde_to_gaus} and \eqref{eq:p_y_tilde_to_gaus} we have that
$$\widetilde{p}_x \dto \N(p_x, \var(\widehat{p}_x)), \qquad \widetilde{p}_y \dto \N(p_y, \var(\widehat{p}_y)).$$
Since they are independent (and so is their limiting distribution), we have that the joint distribution $(\widetilde{p}_x, \widetilde{p}_y)$ also converges to the joint limiting distribution.

The Continuous Mapping Theorem (CMT) states that if $(X_n,Y_n) \xrightarrow{d} (X,Y)$ and and if $g: \mathbb{R}^2 \to \mathbb{R}$ is a continuous function, then $
g(X_n, Y_n) \xrightarrow{d} g(X, Y)$.
In our case, define the function $g: \mathbb{R}^2 \to \mathbb{R}$ by $
g(a, b) = a/b$ for $b \neq 0$. This function is continuous everywhere except at $b = 0$. If $Y$ is continuous (as in our case), then $\mathbb{P}(Y = 0) = 0$.

To conclude, if we denote 
$\widetilde{p}^*= Z_x/Z_y$ where $Z_x \sim  \N(p_x, \var(\widehat{p}_x))$, $Z_y \sim \N(p_y, \var(\widehat{p}_y))$, and $Z_x$ and $Z_y$ are independent, we have that for any $\eta>0$ there exists some $n_0=n_0(\eta)$, such that for any $v$ and $n_x,n_y>n_0$,
\begin{equation}
    |F_{\widetilde{p}}(v)-F_{\widetilde{p}^*}(v)| < \eta.
\end{equation}

This completes the first part of the proof. Now we turn to leverage the results of~\cite{diaz2013existence} stated in Theorem \ref{thm:gauss_approx_ratio} to approximate the distribution of $\widetilde{p}^*$.

Let $\delta_x$ and $\delta_y$ be the coefficients of variation of $\widetilde{D}_x$, and $\widetilde{D}_y$
respectively. Then

$$ \delta^2_x=\var(\widehat{p}_x) {\big /}p_x^2, \quad \delta^2_y=\var(\widehat{p}_y) {\big /} p_y^2.$$
If $\sigma^2=O(n_x)$ then $\delta^2_x = O\left(\frac{1}{n_x}\right)$ since $\var(\widehat{p}_x) = O(\frac{1}{n_x})$, and similarly if $\sigma^2=O(n_y)$ then $\delta^2_y = O\left(\frac{1}{n_y}\right)$.

Let $\lambda=\sqrt{\delta_x^2 + 2\delta_y^2}$. For large enough $n_x$, $n_y$, we have
$\delta_x^2 \leq \lambda \leq 1, \ \delta_y^2 \leq \lambda \leq 1$.

Denote
\begin{equation*}\begin{split}
    V &= p^2 \left( \var(\widehat{p}_x){\big /}p_x^2 + \var(\widehat{p}_y){\big /}p_y^2\right) \\
    &= p^2 \left( \left( \frac{p_x(1-p_x)}{n_x}\right){\bigg /}p_x^2 + \left(\frac{p_y(1-p_y)}{n_y}\right){\bigg /}p_y^2 \right)\\
    &= p^2 \Bigg(\frac{1}{p_xn_x}-\frac{1}{n_x} + \frac{1}{p_yn_y}-\frac{1}{n_y}  \Bigg)
\end{split}\end{equation*}
and set $p^* \sim \N(p, V)$.

    For a normal random variable $Y$ independent of $X$, with small enough $\delta_y$ (which corresponds to large enough $n_y$), Condition \eqref{eq:condition} holds, and we have that for any constant $0<\xi<1$, there exists $n_0(\xi)$, such that for any $n_y>n_0(\xi)$,
    $$ |F_{\widetilde{p}^*}(z)-F_{p^*}(z)| < \xi,$$
    for any $z\in I = [p \pm \frac{\sigma_{\widetilde{p}^*}}{\lambda}]$ where $\sigma_{\widetilde{p}^*}=p\sqrt{\delta^2_x+\delta^2_y}$. Hence, for $z\in I$, using the triangle inequality, we get
    \begin{equation}\label{eq:cum__dist_nor_app}
        \left|F_{\widetilde{p}^*}(z)-F_{p^*}(z) \right| < \eta + \xi.
    \end{equation}
    Note also that 
    \begin{equation}\label{eq:ratio_gaus_dist}
        \frac{\sigma_{\widetilde{p}^*}}{\lambda} = p\frac{\sqrt{\delta^2_x+\delta^2_y}}{\sqrt{\delta_x^2 + 2\delta_y^2}} =   p \sqrt{1 - \frac{1}{\delta_y^2}} \pto p, 
    \end{equation}
thus the limit of $I$ is $\left(0, 2p\right)$.

It follows that the distribution of $\widetilde{p}$ converges to that of $p^*$ (see Equation \eqref{eq:ratio_gaus_dist}) in the interval $\left(0, 2p\right)$.

Define
\begin{equation}
    \widetilde{V} = \widetilde{p}  \cdot \Bigg(\frac{1}{\widetilde{X}}-\frac{1}{n_x} + \frac{1}{\widetilde{Y}}-\frac{1}{n_y} \Bigg)^{1/2}.
\end{equation}

Let
$$ L =  p - \Phi^{-1}(1-\alpha/2) \sqrt{ V}, \ \ U =  p+ \Phi^{-1}(1-\alpha/2) \sqrt{V}$$ 
$$ \tilde L =  p - \Phi^{-1}(1-\alpha/2) \sqrt{\widetilde V}, \ \ \tilde U =  p+ \Phi^{-1}(1-\alpha/2) \sqrt{\widetilde V}$$

$U$ and $L$ lie in the interval where the following hold due to \eqref{eq:cum__dist_nor_app}
\begin{equation}\label{eq:upper_ci}
    |F_{\widetilde p}( U) - F_{p^*}( U) |   \rightarrow 0
\end{equation}
and 
\begin{equation}
    |F_{\widetilde p}( L) - F_{p^*}( L) | \rightarrow 0.
\end{equation}

From Claim \ref{clm:consist_est} we have that $\widetilde{p} \rightarrow p$. From the CMT we have that $\frac{1}{\widetilde{p}_x} \pto \frac{1}{p_x}$ and $\frac{1}{\widetilde{p}_y} \pto \frac{1}{p_y}$. Thus, we can once again use the CMT and get that $\widetilde{V} \pto V$ as $n_x,n_y \rightarrow \infty$. Therefore, $\widetilde {U} \pto U$ and $\widetilde {L} \pto L$. Since $F_{\widetilde p}$ is continuous, we have
\begin{equation}
    |F_{\tilde p}(\tilde U) - F_{\tilde p}( U)| \stackrel{p}{\rightarrow} 0
\end{equation}
and 
\begin{equation}
    |F_{\tilde p}(\tilde L) - F_{\tilde p}( L) | \stackrel{p}{\rightarrow} 0.
    \label{eq:conv_FL_tilde}
\end{equation}

Therefore, exactly as in \citet{lin2024differentially},
\begin{equation*}
    \begin{aligned}
         \Pr &\left( p \in \left(\tilde p - \Phi^{-1}(1-\alpha/2)\sqrt{\widetilde V}, \tilde p + \Phi^{-1}(1-\alpha/2)\sqrt{\widetilde V}\right)\right) \\
        & = \Pr\left( p - \Phi^{-1}(1-\alpha/2)\sqrt{\widetilde V} < \tilde p <p + \Phi^{-1}(1-\alpha/2)\sqrt{\widetilde V} \right) \\
        & = \left(F_{\tilde p}(\tilde U) - F_{\tilde p}( U) \right)+ \left( F_{\tilde p}( U)  - F_{p^*}( U) \right) \\ 
        & \quad 
        -\left(F_{\tilde p}(\tilde L) - F_{\tilde p}( L) \right) - \left( F_{\tilde p}( L)  - F_{p^*}( L) \right) +  \left(F_{p^*}( U) - F_{p^*}( L)\right) .
    \end{aligned}
\end{equation*}
Putting together (\ref{eq:upper_ci}) through (\ref{eq:conv_FL_tilde}) and $F_{p^*}( U) - F_{p^*}( L) = 1- \alpha$, we have
\begin{equation*}
    \lim_{n_x,n_y\rightarrow \infty}\prob \left( p \in \left(\widetilde{p} - \Phi^{-1}(1-\alpha/2)\widetilde{V}, \widetilde p + \Phi^{-1}(1-\alpha/2)\widetilde{V}\right)\right)
      \rightarrow 1- \alpha
\end{equation*}
under the conditions $\Sigma = O(n_x)$ and $\Sigma = O(n_y)$.
\end{proof}

In the proof of Theorem \ref{thm:Valid_CI}, we ignored the perturbation noise. If the perturbation noise has a Gaussian distribution, we can account for both sources and get a more conservative CI that not only approaches the desired CI asymptotically, but also is very close to giving the coverage we seek for small, finite $n$. 

\begin{theorem}[Theorem \ref{thm:Gaussian_Valid_CI}, restated]
    Let $X$ and $Y$ be as in \eqref{eq:bin_rv}. Let $L_x,L_y \overset{\text{i.i.d.}}{\sim} \N(0, \sigma^2)$ independent from $X$ and $Y$. Denote $\widetilde{X} = X+L_x, \ \widetilde{Y} = Y+L_y$, $\widetilde{p}_x \coloneqq \widetilde{X}/n_x$, $\widetilde{p}_y \coloneqq \widetilde{Y}/n_y$, $p=p_x/p_y$, and $\widetilde{p} = \widetilde{p}_x/\widetilde{p}_y$.
    Define
    $$ \widetilde{V} = \widetilde{p}  \cdot \Bigg(\frac{1}{\widetilde{X}}-\frac{1}{n_x} + \frac{1}{\widetilde{Y}}-\frac{1}{n_y} \\ + \sigma^2 \left(\frac{1}{\widetilde{X}^2} + \frac{1}{\widetilde{Y}^2} \right) \Bigg)^{1/2}.$$
    Then, for any $\sigma^2 \in {\mathbb R}^{+}$ such that $\sigma^2 = O(n_x)$ and $\sigma^2 = O(n_y)$, if $n_x$ and $n_y$ both tend to infinity, then for any $0 < \alpha < 1$,
        $$
    \prob\left(p \in \left(\widetilde{p} \pm \Phi^{-1}(1-\alpha/2)\cdot \widetilde{V}\right)\right) \rightarrow 1-\alpha, 
    $$
where $\Phi^{-1}(\cdot)$ is the inverse of the CDF of a standard Gaussian random variable.
\end{theorem}

The proof closely follows the structure of the proof of Theorem \ref{thm:Valid_CI}. In fact, this result can be viewed as a corollary of Theorem \ref{thm:Valid_CI}, as the two differ only by a second-order term. However, we present the full proof here to explicitly highlight how and where the perturbation noise is handled.

\begin{proof}

    Let $X \sim \text{Binomial}(n_x, p_x)$, and define the empirical proportion $\widehat{p}_x := \frac{X}{n_x}$. By the classical Central Limit Theorem (CLT), we have
\[
\sqrt{n_x}(\widehat{p}_x - p_x) \xrightarrow{d} \mathcal{N}(0, \var(\widehat{p}_x)) \quad \text{as } n_x \to \infty,
\]
implying that \(\widehat{p}_x\) is asymptotically normal with mean \(p_x\) and variance $\var(\widehat{p}_x)$.
Let $L_x \sim \N(0, \sigma^2)$ independent of $X$. Define the perturbed estimator $\widetilde{p}_x := \widehat{p}_x + \frac{L_n}{n_x}$. 

From the CMT we have that the sum of independent random variables goes to the sum of their limiting distribution. In our case, $\widehat{p}_x$ goes in distribution to a normal random variable, and the noise is always a random variable. Thus, the limiting distribution of the sum is a normal random variable with mean that is the sum of the means, and variance that is the sum of the variances. It thus follows that  $\widetilde{p}_x$ is asymptotically normal with mean $p_x$ and variance $\var(\widehat{p}_x) + \frac{\sigma^2}{n^2_x}$; that is, for any $\eta > 0$, there exists $n_0 = n_0(\eta)$ such that for all $n_x > n_0$ and all $v \in \mathbb{R}$,
\begin{equation}\label{eq:p_x_tilde_to_gaus2}
\left| \mathbb{P}(\widetilde{p}_x \le v) - G_x(v) \right| < \eta,
\end{equation}
where $G_x$ is the CDF of $\N\left(p_x, \var(\widehat{p}_x)+\frac{\sigma^2}{n^2_x}\right)$.

We can derive the same results for $Y$, and get that for any $\eta > 0$, there exists $n_0 = n_0(\eta)$ such that for all $n_y > n_0$ and all $v \in \mathbb{R}$,
\begin{equation}\label{eq:p_y_tilde_to_gaus2}
\left| \mathbb{P}(\widetilde{p}_y \le v) - G_y(v) \right| < \eta,
\end{equation}
where $G_y$ is the CDF of $\N\left(p_y, \var(\widehat{p}_y)+\frac{\sigma^2}{n^2_y}\right)$.

We now consider the ratio $\widetilde{p} = \widetilde{p}_x/\widetilde{p}_y$. From \eqref{eq:p_x_tilde_to_gaus2} and \eqref{eq:p_y_tilde_to_gaus2} we have that
$$\widetilde{p}_x \dto \N(p_x, \var(\widehat{p}_x)+\sigma^2/n^2_x), \qquad \widetilde{p}_y \dto \N(p_y, \var(\widehat{p}_y)+\sigma^2/n^2_y).$$
Since they are independent (and so is their limiting distribution), we have that the joint distribution $(\widetilde{p}_x, \widetilde{p}_y)$ also converges to the joint limiting distribution.

The Continuous Mapping Theorem (CMT) states that if $(X_n,Y_n) \xrightarrow{d} (X,Y)$ and and if $g: \mathbb{R}^2 \to \mathbb{R}$ is a continuous function, then $
g(X_n, Y_n) \xrightarrow{d} g(X, Y)$.
In our case, define the function $g: \mathbb{R}^2 \to \mathbb{R}$ by $
g(a, b) = a/b$ for $b \neq 0$. This function is continuous everywhere except at $b = 0$. If $Y$ is continuous (as in our case), then $\mathbb{P}(Y = 0) = 0$.

To conclude, if we denote 
$\widetilde{p}^*= Z_x/Z_y$ where $Z_x \sim \N(p_x, \var(\widehat{p}_x)+\sigma^2/n^2_x)$ and $Z_y \sim \N(p_y, \var(\widehat{p}_y)+\sigma^2/n^2_y)$, and $Z_x$ and $Z_y$ are independent, we have that for any $\eta>0$ there exists some $n_0=n_0(\eta)$, such that for any $v$ and $n_x,n_y>n_0$,
\begin{equation}
    |F_{\widetilde{p}}(v)-F_{\widetilde{p}^*}(v)| < \eta.
\end{equation}

This completes the first part of the proof. Now we turn to leverage the results of~\cite{diaz2013existence} stated in Theorem \ref{thm:gauss_approx_ratio} to approximate the distribution of $\widetilde{p}^*$.

Let $\delta_x$ and $\delta_y$ be the coefficients of variation of $\widetilde{D}_x$, and $\widetilde{D}_y$
respectively. Then

$$ \delta^2_x=\left(\var(\widehat{p}_x)+
    \frac{\sigma^2}{n^2_x}\right) {\bigg /}p_x^2, \quad \delta^2_y=\left(\var(\widehat{p}_y)+
    \frac{\sigma^2}{n^2_y}\right) {\bigg /} p_y^2. $$
If $\sigma^2=O(n_x)$ then $\delta^2_x = O\left(\frac{1}{n_x}\right)$ since $\var(\widehat{p}_x) = O(\frac{1}{n_x})$, and similarly if $\sigma^2=O(n_y)$ then $\delta^2_y = O\left(\frac{1}{n_y}\right)$.

Let $\lambda=\sqrt{\delta_x^2 + 2\delta_y^2}$. For large enough $n_x$, $n_y$,
$\delta_x^2 \leq \lambda \leq 1, \ \delta_y^2 \leq \lambda \leq 1$.

Denote
\begin{equation*}\begin{split}
    V &= p^2 \left(\left(\var(\widehat{p}_x)+
    \frac{\sigma^2}{n^2_x}\right){\bigg /}p_x^2 + \left(\var(\widehat{p}_y)+
    \frac{\sigma^2}{n^2_y}\right){\bigg /}p_y^2\right) \\
    &= p^2 \left(\left(\frac{p_x(1-p_x)}{n_x}+
    \frac{\sigma^2}{n^2_x}\right){\bigg /}p_x^2 + \left(\frac{p_y(1-p_y)}{n_y}+
    \frac{\sigma^2}{n^2_y}\right){\bigg /}p_y^2\right) \\
    &= p^2 \Bigg(\frac{1}{p_xn_x}-\frac{1}{n_x} + \frac{1}{p_yn_y}-\frac{1}{n_y}  + \sigma^2 \left(\frac{1}{(p_xn_x)^2} + \frac{1}{(p_yn_y)^2} \right) \Bigg)
\end{split}\end{equation*}
and set $p^* \sim \N(p, V)$.

    For a normal random variable $Y$ independent of $X$, with small enough $\delta_y$ (which corresponds to large enough $n_y$), Condition \eqref{eq:condition} holds, and we have that for any constant $0<\xi<1$, there exists $n_0(\xi)$, such that for any $n_y>n_0(\xi)$
    $$ |F_{\widetilde{p}^*}(z)-F_{p^*}(z)| < \xi,$$
    for any $z\in I = [p \pm \frac{\sigma_{\widetilde{p}^*}}{\lambda}]$ where $\sigma_{\widetilde{p}^*}=p\sqrt{\delta^2_x+\delta^2_y}$. Hence, for $z\in I$, using the triangle inequality, we get
    \begin{equation}\label{eq:cum__dist_nor_app2}
        \left|F_{\widetilde{p}^*}(z)-F_{p^*}(z) \right| < \eta + \xi.
    \end{equation}
    Note also that 
    \begin{equation}\label{eq:ratio_gaus_dist2}
        \frac{\sigma_{\widetilde{p}^*}}{\lambda} = p\frac{\sqrt{\delta^2_x+\delta^2_y}}{\sqrt{\delta_x^2 + 2\delta_y^2}} =   p \sqrt{1 - \frac{1}{\delta_y^2}} \pto p, 
    \end{equation}
thus the limit of $I$ is $\left(0, 2p\right)$.

It follows that the distribution of $\widetilde{p}$ converges to that of $p^*$ (see Equation \eqref{eq:ratio_gaus_dist2}) in the interval $\left(0, 2p\right)$.

Define
\begin{equation}
    \widetilde{V} = \widetilde{p}  \cdot \Bigg(\frac{1}{\widetilde{X}}-\frac{1}{n_x} + \frac{1}{\widetilde{Y}}-\frac{1}{n_y} \\ + \sigma^2 \left(\frac{1}{\widetilde{X}^2} + \frac{1}{\widetilde{Y}^2} \right) \Bigg)^{1/2},
\end{equation}
and the rest of the proof follows exactly as the proof of Theorem \ref{thm:Valid_CI}.
\end{proof}

\subsubsection*{Coverage rate for differentially private ratio of noisy counts with Gaussian noise}\label{subsubsec:gauss_coverage}

The coverage of a CI is the probability (over all sources of randomness) that the CI (which is a random variable) contains the true parameter. Computing the exact coverage of the CI in Theorem \ref{thm:Valid_CI} and Theorem \ref{thm:Gaussian_Valid_CI} requires convolution of a binomial random variable with Gaussian/Laplace noise, so instead we use a numerical simulation. For each set of parameters, we sample $10,000$ pairs of binomial random variables according to Equation~\eqref{eq:bin_rv}, with $n_x=n_y=200$, and we vary $p_x$ and $p_y$. 

We sample Gaussian noise for each binomial draw, with $\varepsilon=0.5$ and $\delta=10^{-4}$, and Laplace noise with $\varepsilon=0.5$. We take the maximum between the perturbed binomial random variable and $1$. For each pair of binomial random variables plus Gaussian or Laplace noise, we compute the CI based on Theorem \ref{thm:Valid_CI}, and Theorem \ref{thm:Gaussian_Valid_CI} (conservative) with $0.95$ confidence level and record whether the actual parameter $p_x/p_y$ is covered by the CI. We then report the average over all the simulations.

\begin{table}[htbp]
\centering
\renewcommand{\arraystretch}{1.1}
\footnotesize
{\centering\large\textbf{Coverage of Gaussian noise addition CI}\par}
\vspace{0.5em}
\begin{tabular}{l|rrrrrrrrr}
\toprule
$p_x$ / $p_y$ &  0.1 &  0.2 &  0.3 &  0.4 &  0.5 & 0.6 &  0.7 &  0.8 & 0.9 \\
\midrule
0.1 & 0.786 & 0.763 & 0.754 & 0.734 & 0.742 & 0.733 & 0.727 & 0.729 & 0.726 \\
0.2 & 0.819 & 0.828 & 0.828 & 0.821 & 0.816 & 0.813 & 0.803 & 0.806 & 0.803 \\
0.3 & 0.833 & 0.830 & 0.839 & 0.848 & 0.846 & 0.846 & 0.836 & 0.836 & 0.833 \\
0.4 & 0.834 & 0.851 & 0.847 & 0.863 & 0.857 & 0.852 & 0.851 & 0.849 & 0.839 \\
0.5 & 0.832 & 0.834 & 0.851 & 0.858 & 0.864 & 0.858 & 0.857 & 0.850 & 0.840 \\
0.6 & 0.828 & 0.834 & 0.858 & 0.864 & 0.863 & 0.854 & 0.849 & 0.848 & 0.834 \\
0.7 & 0.817 & 0.836 & 0.844 & 0.853 & 0.854 & 0.849 & 0.848 & 0.829 & 0.810 \\
0.8 & 0.825 & 0.829 & 0.853 & 0.851 & 0.855 & 0.845 & 0.839 & 0.806 & 0.788 \\
0.9 & 0.817 & 0.828 & 0.847 & 0.847 & 0.852 & 0.832 & 0.817 & 0.782 & 0.721 \\
\bottomrule
\end{tabular}
\vspace{1em}

{\centering\large\textbf{Laplace noise addition CI}\par}
\vspace{0.5em}
\begin{tabular}{l|rrrrrrrrr}
\toprule
$p_x$ / $p_y$ &  0.1 &  0.2 &  0.3 &  0.4 &  0.5 & 0.6 &  0.7 &  0.8 & 0.9 \\
\midrule
0.1 & 0.898 & 0.894 & 0.896 & 0.889 & 0.888 & 0.894 & 0.886 & 0.885 & 0.892 \\
0.2 & 0.913 & 0.922 & 0.921 & 0.923 & 0.920 & 0.919 & 0.919 & 0.913 & 0.915 \\
0.3 & 0.916 & 0.917 & 0.930 & 0.926 & 0.926 & 0.924 & 0.926 & 0.921 & 0.925 \\
0.4 & 0.916 & 0.931 & 0.927 & 0.931 & 0.933 & 0.930 & 0.930 & 0.928 & 0.918 \\
0.5 & 0.918 & 0.925 & 0.930 & 0.932 & 0.933 & 0.927 & 0.927 & 0.931 & 0.929 \\
0.6 & 0.919 & 0.926 & 0.928 & 0.935 & 0.932 & 0.929 & 0.928 & 0.929 & 0.924 \\
0.7 & 0.911 & 0.923 & 0.926 & 0.928 & 0.928 & 0.933 & 0.928 & 0.928 & 0.918 \\
0.8 & 0.918 & 0.924 & 0.932 & 0.928 & 0.929 & 0.930 & 0.924 & 0.915 & 0.916 \\
0.9 & 0.909 & 0.926 & 0.926 & 0.931 & 0.921 & 0.921 & 0.924 & 0.916 & 0.894 \\
\bottomrule
\end{tabular}
\vspace{1em}

{\centering\large\textbf{Conservative Gaussian noise addition CI}\par}
\vspace{0.5em}
\begin{tabular}{l|rrrrrrrrr}
\toprule
$p_x$ / $p_y$ &  0.1 &  0.2 &  0.3 &  0.4 &  0.5 & 0.6 &  0.7 &  0.8 & 0.9 \\
\midrule
0.1 & 0.920 & 0.949 & 0.948 & 0.945 & 0.946 & 0.946 & 0.947 & 0.944 & 0.945 \\
0.2 & 0.915 & 0.946 & 0.947 & 0.949 & 0.950 & 0.944 & 0.945 & 0.949 & 0.947 \\
0.3 & 0.920 & 0.936 & 0.947 & 0.950 & 0.948 & 0.949 & 0.947 & 0.948 & 0.949 \\
0.4 & 0.916 & 0.946 & 0.948 & 0.952 & 0.953 & 0.950 & 0.952 & 0.953 & 0.947 \\
0.5 & 0.919 & 0.938 & 0.947 & 0.950 & 0.952 & 0.950 & 0.953 & 0.954 & 0.948 \\
0.6 & 0.916 & 0.939 & 0.950 & 0.952 & 0.950 & 0.949 & 0.951 & 0.955 & 0.948 \\
0.7 & 0.911 & 0.940 & 0.947 & 0.947 & 0.948 & 0.948 & 0.957 & 0.949 & 0.946 \\
0.8 & 0.912 & 0.939 & 0.950 & 0.948 & 0.947 & 0.948 & 0.951 & 0.949 & 0.952 \\
0.9 & 0.911 & 0.938 & 0.948 & 0.947 & 0.950 & 0.951 & 0.950 & 0.947 & 0.952 \\
\bottomrule
\end{tabular}
\vspace{1em}

{\centering\large\textbf{Conservative Laplace noise addition CI}\par}
\vspace{0.5em}
\begin{tabular}{l|rrrrrrrrr}
\toprule
$p_x$ / $p_y$ &  0.1 &  0.2 &  0.3 &  0.4 &  0.5 & 0.6 &  0.7 &  0.8 & 0.9 \\
\midrule
0.1 & 0.938 & 0.941 & 0.945 & 0.937 & 0.941 & 0.944 & 0.943 & 0.942 & 0.944 \\
0.2 & 0.938 & 0.949 & 0.947 & 0.949 & 0.947 & 0.946 & 0.947 & 0.943 & 0.947 \\
0.3 & 0.937 & 0.941 & 0.950 & 0.946 & 0.949 & 0.945 & 0.951 & 0.947 & 0.948 \\
0.4 & 0.938 & 0.951 & 0.947 & 0.949 & 0.951 & 0.951 & 0.954 & 0.951 & 0.944 \\
0.5 & 0.941 & 0.947 & 0.949 & 0.951 & 0.950 & 0.945 & 0.949 & 0.950 & 0.949 \\
0.6 & 0.940 & 0.947 & 0.950 & 0.950 & 0.951 & 0.949 & 0.948 & 0.951 & 0.948 \\
0.7 & 0.936 & 0.945 & 0.946 & 0.948 & 0.948 & 0.951 & 0.952 & 0.953 & 0.948 \\
0.8 & 0.942 & 0.948 & 0.955 & 0.946 & 0.948 & 0.950 & 0.951 & 0.948 & 0.951 \\
0.9 & 0.934 & 0.949 & 0.950 & 0.949 & 0.944 & 0.947 & 0.951 & 0.950 & 0.949 \\
\bottomrule
\end{tabular}
\vspace{1em}
\caption{Average coverage of 10,000 CIs constructed for 10,000 pairs of binomial random variables drawn according to Equation~\eqref{eq:bin_rv}, with $n_x=n_y=200$. We perturb the counts with both Laplace and Gaussian noise to obtain $(\varepsilon,0)$-DP and $(\varepsilon, \delta=10^{-4})$-DP respectively, and compute CIs using Theorem \ref{thm:Valid_CI} (Private Laplace and Private Gaussian), and according to Theorem \ref{thm:Gaussian_Valid_CI} (Cons Private Gaussian and Laplace); for the Laplace noised counts we plug in the variance of the Laplace.}
\label{tab:coverage-px-py}
\end{table}

\newpage
\section{Further details on the numerical study}\label{App:numerical_study}

\subsection{Sample accuracy}
In Figure \ref{fig:Sample_accuracy}, we consider the sample accuracy (Definition \ref{def:sample_acc}) of several algorithms for differentially private ratio statistics, for three sets of parameters. A much more extensive study obtained very similar results. 

We consider four sets of values (that is, four datasets), $(X,Y)=\{(100,100),(50,100), (100,50), (100,30)\}$. For each set, we compared the sample accuracy $(1-\beta)$ of different algorithms, for varying values of $\varepsilon$. We chose $\alpha=0.1$, and $n_x=n_y=150$. 

\begin{enumerate}
    \item \textbf{Smooth Sensitivity:} The smooth sensitivity algorithm adds Laplace independent noise, with variance proportional to the smooth local sensitivity (see Appendix \ref{App:local_noise_methods}). Thus, once we derive the local sensitivity, we can easily compute $1-\beta$ from the Laplace CDF. We took $\delta=1/n_x=1/150$.
    \item \textbf{Propose-Test-Release:} For propose-test-release detailed in Appendix \ref{App:local_noise_methods} we had to account for the mechanism returning FAIL. We accounted for that as an additional error when compared with other methods. The algorithm also requires a guess for the local sensitivity. We gave the algorithm oracle access to the true sensitivity, that is, the guess is optimized in terms of $1-\beta$. We took $\delta=1/n_x=1/150$.
    \item \textbf{NoisedLog:} We plot the sample accuracy given in Claim \ref{clm:acc_exp_naive} (see Appendix \ref{App:DP_ratio_naive}).
    \item \textbf{LaplaceNoisedCounts:} We plot the sample accuracy given in Claim \ref{clm:acc_Alg_count_no_max} (see Section \ref{sec:noisy_lap_counts}).
\end{enumerate}

\subsection{The bias of LaplaceNoisedCounts}

In Figure \ref{fig:Bias_figure} we quantify the bias of the \textit{LaplaceNoisedCounts$_\varepsilon(X, Y)$} algorithm. We fix $n_x=n_y=150$, and for three sets of parameters $(p_x,p_y) = \{(1/3, 2/3),(1/2,1/2),(2/3, 1/3)\}$, we simulate $20,000$ pairs of Binomially distributed random variables:
$$ X \sim Bin(p_x,n_x), \ \ \ Y \sim Bin(p_y,n_y).$$

For each pair, we use \textit{LaplaceNoisedCounts$_\varepsilon(X, Y)$} to compute a $(\varepsilon,0)$-DP estimate, denoted by $(\widetilde{X}, \widetilde{Y})$. Further denote $X/Y=\widehat{p}_x/\widehat{p}_y=\widehat{p}$ and $\widetilde{X}/\widetilde{Y}=\widetilde{p}_x/\widetilde{p}_y=\widetilde{p}$. 

Then, for each $\varepsilon$, we average the estimates over all $20,000$ independent pairs, and plot the average.

In Figure \ref{fig:Bias_figure} we compare the bias of $\widehat{p}$ (the non-private estimate), $\widetilde{p}$ (the private estimate), and Claim \ref{clm:bias_ratio_counts} about the bias of the private estimate.

\subsection{Confidence interval width}

Following Section \ref{sec:CI}, we compare two non-private methods to construct CIs (for non-private counts), to the methods derived in Theorem \ref{thm:Valid_CI} and in Theorem \ref{thm:Gaussian_Valid_CI}.

We fix $n_x=n_y=150$, and for three sets of parameters $(p_x,p_y) = \{(1/3, 2/3),(1/2,1/2),(2/3, 1/3)\}$, we simulate $10,000$ pairs of Binomially distributed random variables:
$$ X \sim Bin(p_x,n_x), \ \ \ Y \sim Bin(p_y,n_y).$$

For each pair we compute the Classic CI, and the non-private CI. We then perturb the counts with both Laplace and Gaussian noise to obtain $(\varepsilon,0)$-DP and $(\varepsilon, \delta=10^{-4})$-DP, respectively. From the Laplace perturbed counts, we compute the Private Laplace and Cons Private Laplace CI, and from the Gaussian perturbed counts, we compute the Private Gaussian and Cons Private Gaussian CI.

\begin{enumerate}
    \item \textbf{Classic:} This is the classical method in the statistical literature to construct a CI for the relative risk statistic. See Equation~\eqref{eq:conf_int_RR}.
    \item \textbf{Non-Private:} We suggested an alternative approach for constructing CI based on direct normal approximation of the ratio. The method is detailed in Theorem \ref{thm:Valid_CI}.
    \item \textbf{Private Gaussian:} We use the Gaussian-noised-counts in Theorem \ref{thm:Valid_CI}, with $\delta=10^{-4})$.
    \item \textbf{Private Laplace:} We plug the Laplace-noised-counts in Theorem \ref{thm:Valid_CI}.
    \item \textbf{Cons Private Gaussian:} We use the Gaussian-noised-counts in Theorem \ref{thm:gauss_approx_ratio}, where $\sigma^2$ is computed such that the counts are $(\varepsilon, \delta=10^{-4})$-DP according to Theorem \ref{thm:Gaussian_Valid_CI}.
    \item \textbf{Cons Private Laplace:} We use the Laplace-noised-counts in Theorem \ref{thm:gauss_approx_ratio}, with $\sigma^2=2b^2=2(2/\varepsilon)^2$
\end{enumerate}

Then, for each $\varepsilon$, we average the coverage of each CI and its width.

\newpage
\section{Missing proofs}\label{sec:proofs}

\subsection{\textbf{Claim \ref{clm:Lap_ratio_cum_prob}}}

\begin{claim}\label{clm:Lap_ratio_cum_prob}
    Let $X_1 \sim Lap(\mu_1,b)$, and $X_2 \sim Lap(\mu_2,b)$, where $\mu_1,\mu_2 > 0$ are two independent random variables, and let $a\in \mathbb{R}$ be some positive constant.
    
    If $\mu_1/a \geq \mu_2$ then
     \begin{equation*}
        \mathbb{P}(X_1/X_2 < a) = \left(\frac{-1}{2(a+1)(a-1)}\right)\exp\left(\frac{\mu_2a-\mu_1}{b}\right) + \left(\frac{a^2}{2(a+1)(a-1)}\right)\exp\left(\frac{\mu_2-\mu_1/a}{b}\right).
    \end{equation*}

    If $\mu_1/a < \mu_2$ then
    \begin{equation*} \mathbb{P}(X_1/X_2 < a) = 1 + \exp\left(\frac{\mu_1/a-\mu_2}{b}\right) \left( \frac{a^2}{2(a+1)(1-a)}\right)
    - \exp\left(\frac{\mu_1-a\mu_2}{b}\right)\left(\frac{1}{2(a+1)(1-a)} \right). 
    \end{equation*}
\end{claim}

\begin{proof}
If $\mu_1/a \geq \mu_2$, then
    \begin{equation}\begin{split}
        \mathbb{P}(X_1/X_2 < a) &= \int_{-\infty}^{\infty}\mathbb{P}(X_1< a X_2 | X_2=x_2)\mathbb{P}(X_2=x_2) \\
        &= \int_{-\infty}^{\mu_2} \frac{1}{2}\exp \left(\frac{ax_2-\mu_1}{b}\right)\frac{1}{2b}\exp\left(\frac{x_2-\mu_2}{b} \right)dx_2  \\
        & \hspace{2cm} +\int_{\mu_2}^{\mu_1/a} \frac{1}{2}\exp \left(\frac{ax_2-\mu_1}{b}\right)\frac{1}{2b}\exp\left(\frac{\mu_2-x_2}{b} \right)dx_2  \\
        & \hspace{4cm} +\int_{\mu_1/a}^{\infty} \left(1-\frac{1}{2}\exp \left(\frac{\mu_1-ax_2}{b}\right) \right)\frac{1}{2b}\exp\left(\frac{\mu_2-x_2}{b} \right) dx_2\\
        &= \frac{1}{4b} \int_{-\infty}^{\mu_2} \exp \left(\frac{x_2(a+1)-(\mu_1+\mu_2)}{b}\right)dx_2 +
         \frac{1}{4b}\int_{\mu_2}^{\mu_1/a} \exp \left(\frac{x_2(a-1)-(\mu_1-\mu_2)}{b}\right)dx_2  \\
          & \quad + \frac{1}{2b}\int_{\mu_1/a}^{\infty} \exp\left(\frac{\mu_2-x_2}{b} \right) dx_2 - \frac{1}{4b}\int_{\mu_1/a}^{\infty}\exp \left(\frac{x_2(-a-1)+(\mu_1+\mu_2)}{b}\right)dx_2 \\
           &= \frac{1}{4(a+1)}\exp\left(\frac{\mu_2a-\mu_1}{b}\right) +
           \frac{1}{4(a-1)}\left(\exp \left(\frac{\mu_2-\mu_1/a}{b}\right)- \exp \left(\frac{\mu_2a-\mu_1}{b}\right) \right) \\
           & \quad + \frac{1}{2}\exp\left(\frac{\mu_2-\mu_1/a}{b} \right) - \frac{1}{4(a+1)}\exp \left(\frac{\mu_2-\mu_1/a}{b}\right) \\
           &= \exp\left(\frac{\mu_2a-\mu_1}{b}\right)\left(\frac{1}{4(a+1)} - \frac{1}{4(a-1)}\right) + \exp \left(\frac{\mu_2-\mu_1/a}{b}\right) \left(\frac{1}{4(a-1)} + \frac{1}{2} - \frac{1}{4(a+1)} \right) \\
            &= \left(\frac{-1}{2(a+1)(a-1)}\right)\exp\left(\frac{\mu_2a-\mu_1}{b}\right) + \left(\frac{a^2}{2(a+1)(a-1)}\right)\exp\left(\frac{\mu_2-\mu_1/a}{b}\right).
    \end{split}\end{equation}

    If $\mu_1/a \leq \mu_2$, then 
            \begin{equation}\begin{split}
        \mathbb{P}(X_1/X_2 < a) &= \int_{-\infty}^{-\infty}\mathbb{P}(X_1< a X_2 | X_2=x_2)\mathbb{P}(X_2=x_2) \\
        &= \int_{-\infty}^{\mu_1/a} \frac{1}{2}\exp \left(\frac{ax_2-\mu_1}{b}\right)\frac{1}{2b}\exp\left(\frac{x_2-\mu_2}{b} \right)dx_2  \\
        & \hspace{2cm} +\int_{\mu_1/a}^{\mu_2} \left(1-\frac{1}{2}\exp \left(\frac{\mu_1-ax_2}{b}\right)\right)\frac{1}{2b}\exp\left(\frac{x_2-\mu_2}{b} \right)dx_2  \\
        & \hspace{4cm} +\int_{\mu_2}^{\infty} \left(1-\frac{1}{2}\exp \left(\frac{\mu_1-ax_2}{b}\right) \right)\frac{1}{2b}\exp\left(\frac{\mu_2-x_2}{b} \right) dx_2\\
        &= \frac{1}{4b} \int_{-\infty}^{\mu_1/a} \exp \left(\frac{x_2(a+1)-(\mu_1+\mu_2)}{b}\right)dx_2  +\frac{1}{2b}\int_{\mu_1/a}^{\infty}\exp\left(\frac{-|x_2-\mu_2|}{b} \right)dx_2\\
        & \quad -\frac{1}{4b}\int_{\mu_1/a}^{\mu_2} \exp \left(\frac{x_2(1-a)-(\mu_2-\mu_1)}{b}\right)dx_2  - \frac{1}{4b}\int_{\mu_2}^{\infty}\exp \left(\frac{x_2(-a-1)+(\mu_1+\mu_2)}{b}\right)dx_2 \\
           &= \frac{1}{4(a+1)}\exp\left(\frac{\mu_1/a-\mu_2}{b}\right) + 1-\frac{1}{2}\exp\left(\frac{\mu_1/a-\mu_2}{b}\right) \\
            & \quad - \frac{1}{4(1-a)}\left(\exp \left(\frac{\mu_1-a\mu_2}{b}\right)- \exp \left(\frac{\mu_1/a-\mu_2}{b}\right) \right) - \frac{1}{4(a+1)}\exp \left(\frac{\mu_1-a\mu_2}{b}\right) \\
           &= 1 + \exp\left(\frac{\mu_1/a-\mu_2}{b}\right)\left(\frac{1}{4(a+1)} + \frac{1}{4(1-a) } - \frac{1}{2} \right) + \exp\left(\frac{\mu_1-a\mu_2}{b}\right)\left(-\frac{1}{4(a+1)}- \frac{1}{4(1-a)} \right) \\
           &= 1 + \exp\left(\frac{\mu_1/a-\mu_2}{b}\right) \left( \frac{a^2}{2(a+1)(1-a)}\right) - \exp\left(\frac{\mu_1-a\mu_2}{b}\right)\left(\frac{1}{2(a+1)(1-a)} \right).
    \end{split}\end{equation}

\end{proof}

\vfill

\subsection{\textbf{Claim \ref{clm:acc_Alg_count_no_max}, restated}}

\begin{claim}
Consider two natural numbers $X,Y \in {\mathbb N}$, and denote $Z\coloneqq X/Y \in {\mathbb R}^+$. For any $\varepsilon>0$ and $\alpha$ such that $0<\alpha<Z$, if we denote the output of Algorithm \textit{LaplaceNoisedCounts$_\varepsilon(X, Y)$} by $\widetilde{X}$, $\widetilde{Y}$ and take $\widetilde{Z}=\widetilde{X}/\widetilde{Y}$, then $\widetilde{Z}$ is $\left(\alpha, \beta \right)$-sample accurate (see Definition \ref{def:sample_acc}), with
\begin{multline*}
 \beta = \left(0.5 + \frac{0.5}{(Z-\alpha)^2-1} \right)\exp\left(\frac{-\varepsilon\alpha Y}{2(Z-\alpha) }\right)
 +\left(0.5+ \frac{0.5}{(Z+\alpha)^2-1} \right)\exp\left(\frac{-\varepsilon\alpha Y}{2(Z+\alpha) }\right) \\
  \hspace{0.65cm} -\frac{Z^2+\alpha^2-1}{(Z^2+\alpha^2-1)^2 -4\alpha^2 Z^2}\exp\left(\frac{-\varepsilon\alpha Y}{2}\right).
\end{multline*}
\end{claim}

\begin{proof}
    Let there be two natural numbers, $0<X\leq n_x$ and $0<Y\leq n_y$, denote $Z\coloneqq X/Y$, and denote the output of Algorithm \textit{LaplaceNoisedCounts$_\varepsilon(X, Y)$} by $\widehat{Z}$ (a random variable). Let $L_1,L_2 \sim \text{Lap}(2/\varepsilon)$, two independent random variables.
    \begin{equation}\begin{split}\label{eq:acc_count_no_max}
     \mathbb{P}(|\widehat{Z}-Z| \geq \alpha) &= \mathbb{P}\left(
    \Bigg|\frac{X}{Y}-\frac{X+L_1}{Y+L_2} \Bigg| \geq \alpha \right) \\
     &= \mathbb{P}\left(\frac{X+L_1}{Y+L_2} \leq \frac{X}{Y}-\alpha \right) + 1 - \mathbb{P}\left(\frac{X+L_1}{Y+L_2} \leq \frac{X}{Y}+\alpha \right) \\
     &\overset{a}{=} \frac{1}{2(\alpha_1+1)(\alpha_1-1)}\left(\alpha_1^2\exp\left(\frac{Y-X/\alpha_1}{b}\right)-\exp\left(\frac{Y\alpha_1-X}{b}\right)\right) \\
     & \quad + \frac{1}{2(\alpha_2+1)(\alpha_2-1)}\left(\alpha_2^2\exp\left(\frac{X/\alpha_2-Y}{b}\right)-\exp\left(\frac{X-\alpha_2Y}{b}\right)\right)\\
&= \frac{1}{2(\alpha_1+1)(\alpha_1-1)}\left(\alpha_1^2\exp\left(\frac{-\alpha Y}{b(Z-\alpha) }\right)-\exp\left(\frac{- \alpha Y}{b}\right)\right) \\
     & \quad + \frac{1}{2(\alpha_2+1)(\alpha_2-1)}\left(\alpha_2^2\exp\left(\frac{-\alpha Y}{b(Z+\alpha) }\right)-\exp\left(\frac{-\alpha Y}{b}\right)\right)
\end{split}\end{equation}

where (a) follows from Claim \ref{clm:Lap_ratio_cum_prob}, if we denote $\alpha_1 = Z-\alpha$ and $\alpha_2=Z+\alpha$.\\
\\
We note that
$$ \frac{a_1^2}{2(a_1+1)(a_1-1)} = \frac{1}{2}\frac{Z^2-2\alpha Z+\alpha^2}{Z^2-2\alpha Z+\alpha^2-1} = 0.5 + \frac{0.5}{Z^2-2\alpha Z+\alpha^2-1}$$
$$\frac{a_2^2}{2(a_2+1)(a_2-1)} = \frac{1}{2}\frac{Z^2+2\alpha Z+\alpha^2}{Z^2+2\alpha Z+\alpha^2-1} = 0.5+ \frac{0.5}{Z^2+2\alpha Z+\alpha^2-1}$$

Plugging in everything, including $b=2/\varepsilon$, and rearranging, we have

\begin{multline*}
  \mathbb{P}(|\widehat{Z}-Z| \geq \alpha) = \left(0.5 + \frac{0.5}{(Z-\alpha)^2-1} \right)\exp\left(\frac{-\varepsilon\alpha Y}{2(Z-\alpha) }\right) \\ +\left(0.5+ \frac{0.5}{(Z+\alpha)^2-1} \right)\exp\left(\frac{-\varepsilon\alpha Y}{2(Z+\alpha) }\right) \\
   -\frac{Z^2+\alpha^2-1}{(Z^2+\alpha^2-1)^2 -4\alpha^2 Z^2}\exp\left(\frac{-\varepsilon\alpha Y}{2}\right).\qedhere
\end{multline*}

\end{proof}

\newpage
\subsection{Claim \ref{clm:bias_ratio_counts}, restated}
\begin{claim}
Given two numbers $X,Y \in {\mathbb N}$, and $\varepsilon>0$, denote the output of Algorithm \textit{LaplaceNoisedCounts$_\varepsilon(X, Y)$} by $\widetilde{X}$, $\widetilde{Y}$, set $X/Y=Z \in \mathbb{R}^+$, and take $\widetilde{Z}=\widetilde{X}/\max(\widetilde{Y},1)$. Then, for any $\varepsilon>0$ we have
    \small{
    \begin{equation*}\begin{split}
        \mathbb{E}\left[\widetilde{Z}\right] 
       &= X \Bigg( \frac{1}{2}\exp\left(\frac{\varepsilon(1-Y)}{2} \right) +  \frac{\varepsilon}{4}\exp\left(-\frac{\varepsilon Y}{2}\right)  \left(Ei\left(\frac{\varepsilon Y}{2}\right)-  Ei\left(\frac{\varepsilon}{2}\right)\right)    -  \frac{\varepsilon}{4}\exp\left(\frac{\varepsilon Y}{2}\right)Ei\left(-\frac{\varepsilon Y}{2}\right) \Bigg) \\
       &\approx Z \cdot \frac{\varepsilon^2 Y^2}{4} \int_{-\infty}^{0} \frac{-e^u}{u^2 - \left(\frac{\varepsilon Y}{2}\right)^2} \, du,
    \end{split}\end{equation*}
    }
    where $Ei(x)= -\int_{-x}^{\infty}\frac{e^{-t}}{t}dt = \int_{-\infty}^{x}\frac{e^t}{t}dt$.
\end{claim}

\begin{proof}
\begin{equation}\begin{split}
    \mathbb{E}\left[\widehat{Z}\right] 
    &= \mathbb{E}\left[ \frac{X + Lap (2/\varepsilon)}{Max(Y + Lap (2/\varepsilon), 1)}\right] \\
    &= X \mathbb{E}\left[ \frac{1}{Max(Y + Lap (2/\varepsilon), 1)}\right] \\
    &= X \left( \int_{-\infty}^{1}\frac{\varepsilon}{4}\exp\left(-\frac{|x-Y|\varepsilon}{2}\right)dx +  \int_{1}^{\infty} \frac{1}{x}\frac{\varepsilon}{4}\exp\left(-\frac{|x-Y|\varepsilon}{2}\right) dx \right) \\
    &=X \left( \frac{1}{2}\exp\left(\frac{\varepsilon(1-Y)}{2} \right) + \int_{1}^{Y} \frac{\varepsilon}{4x}\exp\left(\frac{(x-Y)\varepsilon}{2}\right) dx + \int_{Y}^{\infty} \frac{\varepsilon}{4x}\exp\left(\frac{(Y-x)\varepsilon}{2}\right) dx \right) \\
    &= X \left( \frac{1}{2}\exp\left(\frac{\varepsilon(1-Y)}{2} \right) + \frac{\varepsilon}{4}\exp\left(-\frac{\varepsilon Y}{2}\right)\int_{1}^{Y} \frac{1}{x}\exp\left(\frac{x\varepsilon}{2}\right) dx + \frac{\varepsilon}{4}\exp\left(\frac{\varepsilon Y}{2}\right)\int_{Y}^{\infty} \frac{1}{x}\exp\left(\frac{-x\varepsilon}{2}\right) dx \right) \\
    &= X \Bigg( \frac{1}{2}\exp\left(\frac{\varepsilon(1-Y)}{2} \right) + \frac{\varepsilon}{4}\exp\left(-\frac{\varepsilon Y}{2}\right)\left(Ei\left(\frac{\varepsilon Y}{2}\right)-Ei\left(\frac{\varepsilon}{2}\right)\right)  - \frac{\varepsilon}{4}\exp\left(\frac{\varepsilon Y}{2}\right)Ei\left(-\frac{\varepsilon Y}{2}\right) \Bigg), \\
\end{split}\end{equation}
where $Ei(x)= -\int_{-x}^{\infty}\frac{e^{-t}}{t}dt = \int_{-\infty}^{x}\frac{e^t}{t}dt$.
\end{proof}

Note that for large enough $Y$ (also compared to $\varepsilon)$, we have

\begin{equation}\begin{split}
    \mathbb{E}\left[\widehat{Z}\right] &\approx X \left(-\frac{\varepsilon}{4}\exp\left(-\frac{\varepsilon Y}{2}\right) Ei\left(\frac{\varepsilon Y}{2}\right) - \frac{\varepsilon}{4}\exp\left(\frac{\varepsilon Y}{2}\right)Ei\left(-\frac{\varepsilon Y}{2}\right)\right)\\
    &=X\frac{\varepsilon}{4} \left(\int_{-\infty}^{-\frac{\varepsilon Y}{2}}\frac{e^{t+\frac{\varepsilon Y}{2}}}{t}dt-\int_{-\infty}^{\frac{\varepsilon Y}{2}}\frac{e^{t-\frac{\varepsilon Y}{2}}}{t}dt \right).\\
\end{split}\end{equation}

Simplifying the integrals separately:

For the first integral substitute \( u = t + \frac{\varepsilon Y}{2} \), and for the second integral substitute \( v = t - \frac{\varepsilon Y}{2} \)

\begin{equation*}
\int_{-\infty}^{-\frac{\varepsilon Y}{2}} \frac{e^{t+\frac{\varepsilon Y}{2}}}{t} \, dt \longrightarrow \int_{-\infty}^{0} \frac{e^u}{u - \frac{\varepsilon Y}{2}} \, du \qquad \int_{-\infty}^{\frac{\varepsilon Y}{2}} \frac{e^{t-\frac{\varepsilon Y}{2}}}{t} \, dt \longrightarrow \int_{-\infty}^{0} \frac{e^v}{v + \frac{\varepsilon Y}{2}} \, dv
\end{equation*}

Now, rewriting the expression:

\begin{equation}\begin{split}
     \mathbb{E}\left[\widehat{Z}\right] &\approx 
     X\frac{\varepsilon}{4} \left( \int_{-\infty}^{0} \frac{e^u}{u - \frac{\varepsilon Y}{2}} \, du - \int_{-\infty}^{0} \frac{e^u}{u + \frac{\varepsilon Y}{2}} \, du \right) \\
     &= X\frac{\varepsilon}{4} \int_{-\infty}^{0} e^u \left( \frac{1}{u - \frac{\varepsilon Y}{2}} - \frac{1}{u + \frac{\varepsilon Y}{2}} \right) du \\
     &= X\frac{\varepsilon}{4} \int_{-\infty}^{0} e^u \left( \frac{(u + \frac{\varepsilon Y}{2}) - (u - \frac{\varepsilon Y}{2})}{(u - \frac{\varepsilon Y}{2})(u + \frac{\varepsilon Y}{2})} \right) du \\
     &= X\frac{\varepsilon}{4} \int_{-\infty}^{0} e^u \left( \frac{\frac{\varepsilon Y}{2} + \frac{\varepsilon Y}{2}}{(u - \frac{\varepsilon Y}{2})(u + \frac{\varepsilon Y}{2})} \right) du \\
     &= X\frac{\varepsilon}{4} \int_{-\infty}^{0} e^u \left( \frac{\varepsilon Y}{(u - \frac{\varepsilon Y}{2})(u + \frac{\varepsilon Y}{2})} \right) du \\
     &= X\frac{\varepsilon^2 Y}{4} \int_{-\infty}^{0} \frac{e^u}{u^2 - \left(\frac{\varepsilon Y}{2}\right)^2} \, du \\
     &= \frac{X}{Y} \times \frac{\varepsilon^2 Y^2}{4} \int_{-\infty}^{0} \frac{e^u}{u^2 - \left(\frac{\varepsilon Y}{2}\right)^2} \, du
\end{split}\end{equation}

When $\varepsilon Y$ is large enough, we have that

$$
\frac{1}{u^2 - \left( \frac{\varepsilon Y}{2} \right)^2} \approx -\frac{1}{\left( \frac{\varepsilon Y}{2} \right)^2}.
$$

The integral becomes:

\[
\int_{-\infty}^{0} \frac{e^u}{u^2 - \left(\frac{\varepsilon Y}{2}\right)^2} \, du \approx -\frac{1}{\left( \frac{\varepsilon Y}{2} \right)^2} \int_{-\infty}^{0} e^u \, du = -\frac{4}{\varepsilon^2 Y^2}.
\]

Thus, as $Y$ grows, the expectation is getting closer to the real value; that is, the bias is decreasing.

\end{document}